\documentclass[11pt,onecolumn]{IEEEtran}
\usepackage{macros}
\usepackage{xcolor}

\usepackage{caption}

\begin{document}

	\title{Minimum mean-squared error estimation with bandit feedback}

\author{
	Ayon Ghosh$^\dagger$,  Prashanth L.A.$^\dagger$, Dipayan Sen$^\dagger$ and Aditya Gopalan$^\ddag$  
	\thanks{ 
		\hspace{-1.5em}$\dagger$ Department of Computer Science and Engineering,
		Indian Institute of Technology Madras, Chennai; 
		E-mail: cs21b013@smail.iitm.ac.in; 
		\{prashla,cs18s012\}@cse.iitm.ac.in,	\\
		$\ddag$ Department of Electrical Communication Engineering,
		Indian Institute of Science, Bangalore,
		E-mail: aditya@iisc.ac.in,
	}
}
 
\maketitle
\begin{abstract}
	We consider the problem of sequentially learning to estimate, in the mean squared error (MSE) sense, a Gaussian $K$-vector of unknown covariance by observing only $m < K$ of its entries in each round.  We propose two MSE estimators, and analyze their concentration properties. The first estimator is non-adaptive, as it is tied to a predetermined $m$-subset and lacks the flexibility to transition to alternative subsets. The second estimator, which is derived using a regression framework, is adaptive and exhibits better concentration bounds in comparison to the first estimator. We frame the MSE estimation problem with bandit feedback, where the objective is to find the MSE-optimal subset with high confidence. We propose a variant of the successive elimination algorithm to solve this problem.  We also derive a minimax lower bound to understand the fundamental limit on the sample complexity of this problem.
\end{abstract}

	
%%%%%%%%%%%%%%%%%%%%%%%%%%%%%%%%%%%%%%%%%%%%%%%%%%%%%%%%%%%%%%%%%%%%%%%%%%%%%%%%

%%%%%%%%%%%%%%%%%%%%%%%%%%%%%%%%%%%%%%%%%%%%%%%%%%%%%%%%%%%%%%%%%%%%%%%%%%%%%%%%
\section{Introduction}
Several real-world applications involve collecting local measurements of a physical phenomenon, and then using the underlying correlation structure to form an estimate  of the physical phenomenon over a wider region. For instance, using sensors to (i) monitor the temperature over a region \cite{guestrin2005near}\; and (ii) detect contamination in a water distribution network \cite{krause08efficient}. Traffic monitoring in a cellular network is another application \cite{utpalcellnetwork}, where the underlying correlation structure plays a major role. In particular, the aim in this application is to collect traffic load measurements from a handful of base stations to form an estimate of the traffic load on all base stations. 

%Some refs to include: 
%Applications: \cite{krause08efficient,deshpande2004model,long2008thermal,utpalcellnetwork}, Correlated bandits: 
 In this paper\footnote{A two-page extended abstract version of this paper which did not include a detailed problem formulation, nor the MSE estimation/optimization algorithms and analyses,  appeared in ICC 2023 \cite{dipayan2023icc}.} we consider a setting where we have a $K$-dimensional Gaussian distribution with a covariance matrix $\Sigma$, and the goal is to find a subset that best captures the underlying correlation structure. A Gaussian model for studying the correlation structure has been shown to be practically viable in \cite{utpalcellnetwork}.
We employ the mean squared error (MSE) objective to capture the underlying correlation structure. For a $m$-subset $A$, the MSE $\psi(A)$ is given by
 \begin{align}
 	\psi(A) = \mathrm{Tr}\left(\Sigma_{A'A'} -\Sigma_{A'A} \big(\Sigma_{AA}\big)^{-1} \Sigma_{AA'}\right), \label{eq:mse-intro}
 \end{align}
where $A'$ is $[K]\setminus A$, and $\Sigma_{AA}, \Sigma_{A'A'}, \Sigma_{A'A}, \Sigma_{AA'}$ are sub-matrices of $\Sigma$ in obvious notation, and $\mathrm{Tr}$ denote the trace function (See Section \ref{sec:problem} for the details). 

We first consider the problem of estimating the MSE of a $m$-subset, say $A$, given a batch of i.i.d. samples for  the sub-matrices $\{\Sigma_{AA}, \Sigma_{AA'}, \Sigma_{A'A}, \Sigma_{A'A'}\}$. This problem is non-adaptive in the sense that the subset $A$ is fixed, the underlying sub-matrices and their associated samples get fixed as well. In other words, using these samples, it is not straightforward to estimate the MSE of a different $m$-subset. 
%change with each sub-matrix entry is pulled equally. An adaptive version of this problem is when we are provided entry-wise estimates of $\Sigma$, with non-uniform sampling.  Such a set of samples facilitates estimation of MSE of any $m$-subset $A$. %In other words, using the same batch of samples, one can estimate MSE of various $m$-subsets. 

From a statistical learning viewpoint, significant progress has been made on the problem of covariance matrix attention (cf. \cite{wainwright-notes}). 
However, the problem of MSE estimation has not received enough attention, and there are no concentration bounds available for the problem of estimating \eqref{eq:mse-intro}, to the best of our knowledge. 
We propose a natural MSE estimator based on sample-averages for the non-adaptive setting. Since the sample average estimator of $\Sigma_{AA}$ may not be invertible, we perform an eigen-decomposition followed by projection of eigenvalues to the positive side. 
Next, we derive a concentration bound for the aforementioned MSE estimator and the bounds that we derive exhibit an exponential tail decay. 

The natural MSE estimator has several disadvantages. First, it is non-adaptive as the estimator is tied to a single $m$-subset, as mentioned before. Second, the concentration bound is not satisfactory both in terms of the constants and the dependence on the number of samples. 
We overcome these issues by designing an efficient unbiased estimator for MSE using a linear regression approach. The resulting estimator is efficient since the MSE of any $m$-subset can be estimated using a given set of samples from the underlying $K$-variate normal distribution. We derive a concentration bound for this estimator. In comparison to the non-adaptive case, this bound features better constants and improved dependence on the number of samples.

We formulate the adaptive estimation problem with bandit feedback in the best-arm identification framework setting \cite{lattimore2020bandit}. To understand the fundamental limit on the sample complexity of this adaptive estimation bandit problem, we derive an information-theoretic lower bound for the special case of $m=2$. We construct a set of covariance matrices that are rich enough to include the least favorable instance for any bandit algorithm. We establish the lower bound using the well-known standard change of measure argument by constructing problem transformations based on the aforementioned set of covariance matrices, but the technical steps require significant deviations in terms of algebraic effort. Moreover, the setting we consider involve sampling more than one arm, which is strictly necessary for estimating the underlying correlation. This sampling change implies additional effort in computing certain KL-divergences, which are then related to the sub-optimality gap in MSEs. 

Finally, we design a successive elimination-type algorithm \cite{even2002pac} to cater to the adaptive estimation of MSE using bandit feedback. This algorithm uses the efficient MSE estimator discussed above, and the concentration bound we derived for this estimator is utilized to form confidence width in the successive elimiation algorithm.
We derive an upper bound on the sample complexity of this algorithm.  

\textit{Related work.} 
Previous works such as \cite{BubeckLiu,boda2019correlated,gupta2020correlated} feature bandit formulations where the underlying correlation structure appears in the objective. 
In \cite{BubeckLiu}, the aim is to find the maximum correlated subset, i.e., a set that has highly correlated members. In contrast, our goal is to find a subset that best captures information about other, as quantified by the MSE objective. In addition, unlike \cite{BubeckLiu}, we do not assume unit variances in the underlying model. Next, in \cite{boda2019correlated}, which is the closest related work, the authors propose an MSE-based objective for a simplified version of the problem where the goal is to find an arm (or $1$-subset) that is most correlated to the remaining $K-1$ arms in the MSE sense. Our problem formulation is more general as we consider MSE of $m$-subsets, with $1\le m \le K$. This generalization leads to bigger technical challenges in MSE estimation and concentration, as well as in the lower bound analysis.
Finally, in \cite{gupta2020correlated}, the authors assume that the arms are correlated through a latent random source, and the objective is to identify the arm with the highest mean. In 
\cite{ErraqabiLVBL17}, the authors study the impact of correlation on the regret, while featuring a regular bandit formulation, i.e.. of identifying the arm with the highest mean.

The rest of the paper is organized as follows:
In Section~\ref{sec:problem}, we formally define the notion of MSE. 
In Section~\ref{sec:naest}, we describe MSE estimation in the non-adaptive setting, while in Section~\ref{sec:effest} we describe our efficient MSE estimation scheme and provide improved concentration bounds. 
In Section \ref{sec:successive-elim}, we present a variant of successive elimination algorithm for solving the adaptive estimation problem in the fixed confidence setting of the best arm identification (BAI) framework. 
In Section \ref{sec:lb}, we present a minimax lower bound on the sample complexity of the adaptive MSE estimation problem with bandit feedback in a BAI setting. 
In Section~\ref{sec:pf}, we present  detailed proofs of the theoretical results in Sections~\ref{sec:naest}--\ref{sec:successive-elim}. 
% In Section~\ref{sec:exp}, we present numerical experiments for MSE estimation and the bandit application.
Finally, in Section~\ref{sec:conclusions} we provide our concluding remarks.
%%%%%%%%%%%%%%%%%%%%%%%%%%%%%%%%%%%%%%%%%%%%%%%%%%%%%%%%%%%%%%%%%%%%%%%%%%%%%%%%
\section{Preliminaries}
\label{sec:problem}
We consider a jointly Gaussian $K$-vector $X=(X_1,\ldots,X_K)$, with mean zero and covariance matrix $\Sigma \triangleq  \E[X\tr X]$: 
\begin{align}
	\Sigma  = \! \left[ \! \begin{array}{c c c c}
		\sigma^2_1 & \rho_{12}\sigma_1\sigma_2  & \ldots & \rho_{1K}\sigma_1\sigma_K \\
		\rho_{12}\sigma_1\sigma_2  & \sigma^2_2 & \ldots  & \rho_{2K}\sigma_2\sigma_K \\
		\vdots & \vdots & \ddots & \vdots \\
		\rho_{1K}\sigma_1\sigma_K  & \rho_{2K}\sigma_2\sigma_K & \ldots & \sigma^2_K   
	\end{array} \! \right], \!
	\label{eq:gauss-covar}
\end{align}
where $\sigma^2_i$, $i \in [K]$ is the variance of arm $i$ and $\rho_{ij}$, $i,j=1,\ldots,K, \ i \neq j,$ the correlation coefficient between arms $i$ and $j$. Here $[n] =\{1,\ldots,n\}$, for any natural number $n$.

Let $\A$ denote the set of subsets of $[K]$ with cardinality $m$.
The mean-squared error (MSE) for a given subset \\$A=\{i_1,\ldots,i_m\} \in \A$ is defined as 
\begin{align}
	\psi(A)  
	&\triangleq\sum_{j=1}^{K} \E[(X_j - \E[X_j|X_{i_1},\ldots,X_{i_m}])^2].\label{eq:mse-def}
\end{align}	
As shown by \cite{hajek2009notes}, the above definition is equivalent to 
\begin{align}
	\mse(A) = \mathrm{Tr} \left (\Sigma_{A'A'} -\Sigma_{A'A} \big(\Sigma_{AA}\big)^{-1} \Sigma_{AA'}\right),
	\label{eq:mse-basic}
\end{align}
where $\mathrm{Tr}$ denotes the trace function, $A' = [K]\setminus A$ is the complement of $A$, $\Sigma_{A'A'}$ (resp. $\Sigma_{AA}$) is the covariance matrix, which is obtained by restricting $\Sigma$ to the set $A'$ (resp. $A$).

%
%\begin{align}
%	\Sigma =
%	\begin{bmatrix}
%		1 & \rho & \rho & \rho \\
%		\rho & 1 & \rho^2 & \rho^2  \\
%		\rho & \rho^2 & 1 & \rho^3 \\
%		\rho & \rho^2 & \rho^3 & 1
%	\end{bmatrix} \label{eq:cov_matrix_4}
%\end{align}
%The mean-squared error (MSE) for a given subset $A=\{k,l\}$, for some $k,l \in [K], k\ne l$, is defined as 
%\begin{align*}
%	\psi(A)  
%	&=\sum_{j=1}^{4} \E[(X_j - \E[X_j|X_k,X_l])^2] \\
%	&=\sum_{j\neq k,l} \sigma_{j}^2 \Bigg(1-\frac{1}{(1-\rho_{kl}^2)} \Bigg(\rho_{jk}^2+\rho_{jl}^2+2\rho_{kl}\rho_{jk}\rho_{jl}\Bigg)\Bigg).
%\end{align*}  
%
%The first problem we consider is estimation of $\psi(A)$ for a given subset $A$. 
%
%The MSEs corresponding to the individual pairs in the case of $K = 4$ are as follows, 
%\begin{align*}
%	\mse_{12} & = 2- \frac{2}{1- \rho^2} \left(\rho^2 + 3 \rho^4\right) \\
%	\mse_{13} = \mse_{14} & = 2 - \frac{1}{1- \rho^2} \left(2\rho^2 + 3 \rho^4 + 2\rho^5 + \rho^6 \right) \\
%	\mse_{23} = \mse_{24} & = 2 - \frac{1}{1-\rho^4} \left(2 \rho^2 + 3 \rho^4 + \rho^6 + 2 \rho^7\right) \\
%	\mse_{34} & = 2 - \frac{1}{1- \rho^6} \left(2 \rho^2 + 2 \rho^4 + 2 \rho^5 +2 \rho^7\right)
%\end{align*}
%We can infer from the MSE expressions easily that the pair $\{1,2\}$ has the least MSE.

In next two sections, we describe the MSE estimation problem in the non-adaptive and adaptive settings, respectively. Subsequently, we present lower and upper bounds for the correlated bandit problem with an MSE objective.
%%%%%%%%%%%%%%%%%%%%%%%%%%%%%%%%%%%%%%%%%%%%%%%%%%%%%%%%%%%%%%%%%%%%%%%%%%%%%%%%
\section{Non-adaptive MSE estimation}
\label{sec:naest}
%The mean-squared error (MSE) for a $m$-subset $A=\{i_1,\ldots,i_m\} \in \A$ is 
%\begin{align*}
%	\mse(A) = \mathrm{Tr} \left (\Sigma_{A'A'} -\Sigma_{A'A} \big(\Sigma_{AA}\big)^{-1} \Sigma_{AA'}\right)
	%\label{eq:mse-basic}
%\end{align*}
From \eqref{eq:mse-basic}, it is apparent that one can form an estimate of $\mse(A)$ using estimates of the sub-matrices  $\Sigma_{AA},\Sigma_{A'A},\Sigma_{AA'}$, and $\Sigma_{A'A'}$. In the non-adaptive setting, we are given i.i.d.\;samples for each of the sub-matrices $\Sigma_{AA}, \Sigma_{A'A'}, \Sigma_{A'A}, $ and $\Sigma_{AA'}$, for a given subset $A$. Using these samples from the underlying multivariate Gaussian distribution, we form the sample covariance matrices $\widehat \Sigma_{AA}$, $\widehat \Sigma_{AA'}$, $\widehat \Sigma_{A'A}$, and $\widehat \Sigma_{A'A'}$ to estimate the aforementioned four sub-matrices. 
% The goal is to estimate $\psi(A)$ using these sample covariance matrices. All the entries of a particular sub-matrix (among the four mentioned above) are sampled equally in this setting. 
%\subsection{MSE estimation: Basic form}
 
%\todod{Define $\widehat \Sigma_{AA}$ with $n_{AA}$ samples. Then, say the other sub-matrices are formed in a similar fashion. In particular, $n_{AA'},n_{A'A'},n_{A'A}$ samples are used for forming $\widehat \Sigma_{AA},\widehat \Sigma_{AA},\widehat \Sigma_{AA}$, respectively.}
The `sample-average' estimator $\widehat \Sigma_{AA}$ is not guaranteed to be invertible (though it is positive definite with high probability), while MSE estimation requires an estimate of $\Sigma_{AA}^{-1}$. To handle invertibility, we form the matrix $\widehat{\Sigma_{AA}}^{+}$ by performing an eigen-decomposition of $\widehat \Sigma_{AA}$, followed by a projection of eigenvalues to the positive side.  Formally, for $i=1,\ldots,m$, let $\hat{\lambda}_i$ denote the eigenvalue of  $\widehat \Sigma_{AA}$, with corresponding eigenvector $v_i$. The estimator $\widehat{\Sigma}^+_{AA}$ is defined by%%make the estimator positive semi-definite, we use the following modified estimator as mentioned in \citep{Cai_2010},
\begin{align}
	& \widehat{\Sigma}_{AA}^{+} \triangleq \sum\limits_{i = 1}^{m} \hat{\lambda}_{i}^{+} v_{i} v_{i}^\intercal \label{eq:sigma-hat-plus},
\end{align}
where $\hat{\lambda}_{i}^{+} = \begin{cases}
    \hat{\lambda}_{i} & \textrm{ if } |\hat{\lambda}_{i}| \ge \zeta,\\
    \zeta & \textrm{otherwise},
\end{cases}$
% \max (\hat{\lambda}_{i}, -\hat{\lambda}_{i})$, 
for $i=1,\ldots,m$. 
It is easy to see that $\widehat{\Sigma}_{AA}^{+}$ is positive definite.

The MSE $\psi(A)$ associated with set $A$ is then estimated as follows:
\begin{align}
	\widehat{\mse}(A) \triangleq \mathrm{Tr} \left(\widehat{\Sigma_{A'A'}} - \widehat{\Sigma_{A'A}} \big(\widehat{\Sigma_{AA}}^{+}\big)^{-1} \widehat{\Sigma_{AA'}}\right).\label{eq:mse-est-basic}
\end{align}
Next, we proceed to analyze the concentration properties of the estimator defined above.
For the sake of analysis, we make the following assumptions:
\begin{assumption}
	\label{ass:variance}
$0 < l = \underset{i}{\min} \; \sigma_{i}^2 , \sigma_i^2 \le 1$ for $i=1,\ldots,K$.
\end{assumption}

\begin{assumption}
	\label{ass:eigenvalue}
$\max\left(||\Sigma_{AA}||_{2},||\Sigma_{A'A'}||_{2} \right) \leq M_{0},$ and $||\Sigma_{AA}^{-1}||_{2} \leq \frac1{M_1}$, where $|| \cdot ||_2$ is the operator norm.
\end{assumption}
Assumption \ref{ass:variance} is used for the simpler $1$-subset MSE estimation by \cite{boda2019correlated}, while \ref{ass:eigenvalue} is common in the analysis of covariance matrix estimates (cf. \cite{Cai_2010}).
We now present a concentration bound for the MSE estimator \eqref{eq:mse-est-basic}.
\begin{proposition}[\textbf{\textit{MSE concentration: Non-adaptive case}}]
	\label{prop:nonadaptive-mse-conc}
	Assume \ref{ass:variance} and \ref{ass:eigenvalue}. 
 Let $n_{AA},n_{A'A},n_{AA'}, n_{A'A'} $ denote the number of samples used to form $\widehat\Sigma_{AA}, \widehat\Sigma_{A'A}, \widehat\Sigma_{AA'}$, and $\widehat\Sigma_{A'A'}$, respectively. 
 Set the projection parameter $\zeta$ in \eqref{eq:sigma-hat-plus} as follows:
 \[\zeta= M_{0} \min \bigg( \sqrt{\frac{m+\log(\frac{1}{\delta})}{n_{AA}}},\frac{m+\log(\frac{1}{\delta})}{n_{AA}}\bigg).\]
where $\delta \in (0,1) $ denotes the confidence width. Let $n'=\min\left(n_{AA}, n_{A'A'}, n_{AA'}, n_{A'A}\right)$.

 Then, for any $0 < \epsilon < \eta \triangleq  \min\left(2K,\lambda_{min}(\Sigma_{AA})\right)$\footnote{$ \lambda_{\min}(AA)$ denotes the smallest eigenvalue of the matrix $AA$. }, the MSE estimate $\widehat{\mse}(A)$ defined by \eqref{eq:mse-est-basic} satisfies
	\begin{align}
		& \mathbb{P}\left(|\widehat{\mse}(A)-\mse(A)| \geq \epsilon \right)  \leq  
		\underbrace{ C_0
		\exp \left( \frac{-n'}{m K^2 (1+\eta)^3}  \min \left(\frac{\epsilon}{12 G_{0}},\frac{\epsilon^2}{G_{0}^2}\right)\right)}_{(I)} 
		+ \underbrace{2 m K \exp \left( - \frac{n' \epsilon^2}{72 \; C_{3}^2 m^2 K^4 (1+\eta)^3} \right)}_{(II)} , \label{eq:non-adaptive-conc}
	\end{align}	
where  $ c= \frac1{[\lambda_{\min}(\Sigma_{AA})- \epsilon]}$,
$C_0 = \left[ 13 \text{ mK} + e^{m + K}\right],$ 
    $C_{1} = \frac{108 \sqrt{2}}{l} \left(c + \frac{1}{M_{1}}\right),$
    $C_{2} = \bigg(\frac{3 c M_{0}}{M_1}\bigg), $
    $C_{3} = \left(c + \frac{1}{M_{1}}\right),$ and
    $G_{0} = \max \Big(\sqrt{m (K - m)^2} (1 + \eta) \; C_{1}, \sqrt{m (K - m)^2} (1 + \eta) \; C_{2}, 3 (K - m) M_{0} \Big).$
\end{proposition}
\begin{proof}
	See Section \ref{sec:proof-nonadaptive-conc}. 
\end{proof}
%\todod{you have modified the expression of the bound changing $(K - m)$ to $K$ , right? then $G_{0}$ will change too}
In the result above, we have $\epsilon \le 2K$, and this constraint is not restrictive since the MSE $\psi(A) \le K$ for any subset $A$ as a consequence of \ref{ass:variance}. 

	To understand the terms (I) and (II) in \eqref{eq:non-adaptive-conc}, we have to look at the following decomposition of the MSE estimation error:	
	\begin{align}
		 \widehat{\mse}(A)-\mse(A) &=  \Tr \left(\left(\widehat{\Sigma_{A'A'}} - \Sigma_{A'A'}\right) -
		\left(
		\widehat{\Sigma_{A'A}} \left(\widehat{\Sigma_{AA}}^{+}\right)^{-1}  \left(\widehat{\Sigma_{AA'}} - \Sigma_{AA'}\right) \right. \right. \nonumber\\
  &\left. \left. \qquad \quad  + \widehat{\Sigma_{A'A}} \left(\left(\widehat{\Sigma_{AA}}^{+}\right)^{-1} - \Sigma_{AA}^{-1}\right) \Sigma_{AA'} + \left(\widehat{\Sigma_{A'A}} - \Sigma_{A'A}\right)\Sigma_{AA}^{-1} \Sigma_{AA'}\right)
		\right). \label{eq:s1_first_main_}
	\end{align}
The first and third terms on the RHS above relate to estimation of a covariance matrix and its inverse. These terms lead to the term (I) in the bound \eqref{eq:non-adaptive-conc} above. On the other hand the second and fourth terms on the RHS above relate to concentration of sample standard deviation and sample correlation coefficient, in turn leading to the term (II) in the bound \eqref{eq:non-adaptive-conc}.

The estimator \eqref{eq:mse-est-basic} has two limitations. First, the estimation scheme underlying \eqref{eq:mse-est-basic} is tied to a particular subset, making it non-adaptive since moving between subsets would entail fresh sampling. 
%Further, the conproposed approach from Proposition \eqref{prop:nonadaptive-mse-bound} for estimating \( \psi(A) \), based on the estimation of the covariance matrix \( \Sigma \), is inefficient. The estimation process is overly cumbersome and does not leverage available structure in a way that reduces sample complexity.
Second, the constant \( C_0 \)  appearing in the bound of Proposition \ref{prop:nonadaptive-mse-conc} grows exponentially with \( m \) and \( K \), making the upper bound unrealistically large. Furthermore, the bound in Proposition \ref{prop:nonadaptive-mse-conc}  exceeds one when the sample size \( n \) is smaller than \( (m+K)mK^2 \), implying the bound is applicable for large sample sizes only. 
We overcome these limitations using an efficient MSE estimator described in the next section.
%This discrepancy introduces a significant gap between the theoretical and practical sample sizes required.
%%%%%%%%%%%%%%%%%%%%%%%%%%%%%%%%%%%%%%%%%%%%%%%%%%%%%%%%%%%%%%%%%%%%%%%%%%%%
%%%%%%%%%%%%%%%%%%%%%%%%%%%%%%%%%%%%%%%%%%%%%%%%%%%%%%%%%%%%%%%%%%%%%%%%%%%%
%%%%%%%%%%%%%%%%%%%%%%%%%%%%%%%%%%%%%%%%%%%%%%%%%%%%%%%%%%%%%%%%%%%%%%%%%%%%
%%%%%%%%%%%%%%%%%%%%%%%%%%%%%%%%%%%%%%%%%%%%%%%%%%%%%%%%%%%%%%%%%%%%%%%%%%%%
%%%%%%%%%%%%%%%%%%%%%%%%%%%%%%%%%%%%%%%%%%%%%%%%%%%%%%%%%%%%%%%%%%%%%%%%%%%%%%%%
%%%%%%%%%%%%%%%%%%%%%%%%%%%%%%%%%%%%%%%%%%%%%%%%%%%%%%
\section{Efficient MSE estimation}
\label{sec:effest}

We first show that the problem of getting an efficient estimator for \( \psi(A) \) for a \( K \)-arm problem can be related to a linear regression problem. In the regression framework, we characterize the distribution of the residual sum of squares, which in turn leads to an unbiased estimator for \( \psi(A) \). Subsequently, we analyze the concentration properties of this estimator.

Recall that we have a $K$-variate Gaussian vector with mean zero and covariance \( \Sigma \). Focusing on an arbitrary arm \( j \), we analyze its distribution with respect to an \( m \)-subset \( A \). Consider the random vector  
\[
X = \begin{bmatrix} x_j \\ x_A \end{bmatrix},
\]
where \( x_j \) is a scalar random variable corresponding to arm \( j \), and \( x_A \) is an \( m \)-dimensional random vector associated with the subset \( A \). The distribution of \( X \) is multivariate Gaussian (Section 3.4 of \cite{hajek2009notes}) and can be expressed as  
\[
X \sim \mathcal{N} \left(
\begin{bmatrix} 0 \\ 0 \end{bmatrix},
\begin{bmatrix} \Sigma_{jj} & \Sigma_{jA} \\ \Sigma_{Aj} & \Sigma_{AA} \end{bmatrix}
\right),
\]
where the covariance matrix components have dimensions \( \Sigma_{jj}: 1 \times 1 \), \( \Sigma_{jA}: 1 \times m \), \( \Sigma_{Aj}: m \times 1 \), and \( \Sigma_{AA}: m \times m \), respectively. From properties of multivariate Gaussian distributions (cf. Theorem 3 of \cite{holt2023bayesian}), the conditional distribution of \( x_j\) given \({x_A} \) is  given by
\[
x_j \mid x_A \sim \mathcal{N}(\mu_j, \gamma_j^2),
\textrm{ where  }
\mu_j = \Sigma_{jA} \Sigma_{AA}^{-1} x_A, \quad \gamma_j^2 = \Sigma_{jj} - \Sigma_{jA} \Sigma_{AA}^{-1} \Sigma_{Aj}.
\]
Since \( \psi(A) = \sum_j \gamma_j^2 \), estimating \( \gamma_j^2 \) for all \( K \) arms and summing them yields an estimator for \( \psi(A) \).
We shall see that such an approach leads to an efficient estimator that can work for any $m$-subset $A$ using the same set of samples from the underlying $K$-variate Gaussian distribution.  

\textbf{MSE estimation using a regression framework.}  
 We now focus on estimating the parameter $\gamma_j^2$, for any $j=1,\ldots,K$. 
 Given \( n \) i.i.d. observations of \( x_j \) and \({x_A} \), let \( X_j \in \mathbb{R}^{n \times 1} \) represent the \( n \)-dimensional vector of observed values for \( x_j \), and let \( X_A \in \mathbb{R}^{n \times m} \) denote the \( n \times m \) matrix containing the corresponding observations for \( {x_A} \).  
The relationship between \( X_j \) and \( X_A \) is given by
\[
X_j = X_A \beta_j + \epsilon_j, \textrm{ where } \epsilon_j \sim \mathcal{N}(0, \gamma_j^2 I_n).
\]
The Ordinary Least Squares (OLS)  estimator is given by  
\[
\hat{\beta}_j = (X_A^\top X_A)^{-1} X_A^\top X_j.
\]
The OLS estimator is well defined and the matrix \( X_A^\top X_A \) is almost surely invertible, when the random variables \( X_{ij} \) (entries of the matrix \( X_A \)) have a bounded density, cf. \cite[Section 6]{carpentier2022estimation}. In our setting, the latter requirement is easily met since the underlying distribution is assumed to be multivariate Gaussian.
% under Assumption \ref{ass:A3}, ensuring that the matrix \( X_A^\top X_A \) is almost surely invertible.

%For the analysis, we make the following assumption to ensure invertibility issues in our regression do not arise.  
%
%\begin{assumption}
%    \label{ass:A3}
%    The random variables \( X_{ij} \) (entries of the matrix \( X_A \)) have a density with respect to the Lebesgue measure, bounded by a constant \( K > 0 \).
%\end{assumption}

%which holds with probability one.

We define the projection of \( X_j \) onto the column space of \( X_A \) as
\[
\hat{X_j} = P_A X_j, \textrm{ where } P_A = X_A (X_A^\top X_A)^{-1} X_A^\top.
\]
The residual error \( R_j = \hat{X_j} - X_j = X_A (\hat{\beta_j} - \beta_j) - \epsilon_j \). Thus, the estimator \( \hat{\beta_j} \) can be rewritten as
\[
\hat{\beta_j} = \beta_j + (X_A^\top X_A)^{-1} X_A^\top \epsilon_j.
\]
From the foregoing, \( \hat{\beta_j} - \beta_j = (X_A^\top X_A)^{-1} X_A^\top \epsilon_j \), and the residual can be rewritten as 
\[
R_j = -(I_n - P_A)\, \epsilon_j.
\]

The matrix \( (I_n - P_A) \) is idempotent, symmetric, and has rank \( n - m \), with trace \( \text{Tr}(I_n - P_A) = n - m \). Therefore, the quadratic form \( \frac{R_j^\top R_j}{\gamma_j^2} \) has a scaled chi-squared distribution, i.e., satisfies
\[
\frac{R_j^\top R_j}{\gamma_j^2} \sim \chi^2_{n - m}, \quad \mathbb{E}\left( \frac{R_j^\top R_j}{\gamma_j^2} \right) = n - m,
\]
where $\chi^2_{n - m}$ denotes a chi-squared distribution with $(n-m)$ degrees of freedom.
The expression for expectation above is equivalent to 
\[
\mathbb{E}\left( \frac{\|R_j\|^2}{n - m} \right) =\mathbb{E}\left( \frac{\|X_j - P_A X_j\|^2}{n - m} \right) = \gamma_j^2.
\]
The expression above in conjunction with the fact that $\psi(A)= \sum_{j=1}^K \gamma_j^2$ leads to the following estimator 
\begin{equation}
\label{eq:mse_cleverest}
\widetilde{\psi}(A) = \sum_{j=1}^K \frac{X_j^\top (I_n - P_A) X_j}{n - m}.
\end{equation}

It is easy to see that \( \widetilde{\psi}(A) \) is an unbiased estimator for \( \psi(A) \). We now present an improved concentration bound  for this estimator. 

\begin{proposition}[\textbf{\textit{MSE concentration: Adaptive case}}]
\label{prop:adaptive-mse-conc}
Assume \ref{ass:variance}
For a given $m$-subset $A$, form the MSE estimator $\widetilde{\psi}(A)$ from \eqref{eq:mse_cleverest}, after drawing $n$ i.i.d. samples from the underlying $K$-variate Gaussian distribution. Then, for any $\epsilon>0$, we have
\begin{align}
\mathbb{P} \left( \left| \widetilde{\psi}(A) - \psi(A) \right| > \epsilon \right) &\leq 2K \exp \left[ 
    - \frac{(n-m)}{8} \min \left( 
    \frac{9\epsilon^2}{K^2}, 
    \frac{3\epsilon}{K} 
    \right) 
    \right], \label{eq:fin_bound}
\end{align}
where $c_{1}$ is a universal constant.
\end{proposition}
\begin{proof}
	See Section \ref{sec:proof-adaptive-conc}. 
\end{proof}

\begin{remark}[Comparison with the non-adaptive case]
The bound in \eqref{eq:fin_bound}  features several improvements over the one in \eqref{prop:adaptive-mse-conc}.
First,  the constant \( 2K \) multiplying the exponential factor is a significant improvement over the constant \( C_0 = 13mK + e^{m+K} \) in \eqref{eq:non-adaptive-conc}. 
Further, the bound in \eqref{eq:fin_bound} does not exceed  one when \( n < (m+K)mK^2 \), unlike \eqref{eq:non-adaptive-conc}. 
Second, the bound in \eqref{eq:fin_bound} leads to a better sample complexity.
The previous bound required a sample size of \\\( n' = \min(n_{AA'}, n_{A'A}, n_{AA}, n_{A'A'}) \), with separate samples needed for estimating each covariance matrix in \eqref{eq:mse-def}. In contrast, the current approach simplifies the sample complexity to \( nK \), requiring just \( n \) samples for each of the \( K \) arms. This streamlined approach reduces the total sample requirement significantly.
Third, Proposition \ref{prop:adaptive-mse-conc} does not require \ref{ass:eigenvalue}, unlike Proposition \ref{prop:nonadaptive-mse-conc}. Assumption \ref{ass:eigenvalue} is restrictive, limiting the applicability of the bound in \eqref{eq:non-adaptive-conc}. 
\end{remark}

%%%%%%%%%%%%%%%%%%%%%%%%%%
%%%%%%%%%%%%%%%%%%%%%%%%%%%%%%%%%%%%%%%%%%%%%%%%%%%%%%%%%%%%%%%%%%%%%%%%%%%%%%%%
\section{Bandits with MSE objective:  Upper bound}
\label{sec:successive-elim}
We consider the fixed confidence variant of the best-arm identification framework \cite{lattimore2020bandit}. In this setting, the interaction of a bandit algorithm with the environment is given below.
\begin{figure}[H] \label{fig:fig1}
	\fbox{
		\begin{minipage}{0.9\textwidth}
			{\bfseries Adaptive estimation with bandit feedback}\ \\[0.5ex]
			{\bfseries Input}: set of $m$-subsets $\A$. \\
			{\bfseries For all $t=1,2,\ldots,$ repeat} 
			\begin{enumerate}%[(1),leftmargin=+0.2in]
				\item Observe a sample $K$-vector from the underlying multivariate Gaussian distribution.
				\item Choose to continue, or stop and output an $m$-subset.
			\end{enumerate}
		\end{minipage}
	}
	\label{fig:flow}
\end{figure}
Notice that the bandit interaction involves sampling the $K$-vector and not a particular $m$-subset in each round.
The rationale behind such a formulation is that given the $K$-vector, using the efficient estimator from Section \ref{sec:effest}, we can estimate the MSE of any subset $A \in \A$. The non-adaptive estimator from Section \ref{sec:naest} does not facilitate such an interaction, as it requires sampling sub-matrices of the covariance matrix for a given subset.

A subset that has the lowest MSE is considered optimal, i.e., 
\[ A^*  \in \argmin\limits_{A \in \A} \mse(A). \]
The aim in this setting is to devise an algorithm that outputs the best $m$-subset with high probability, while using a low number of samples.  More precisely, for a given confidence parameter $\delta\in (0,1)$, an algorithm is $\delta$-PAC if it stops after $\tau$ rounds, and outputs a set $A_\tau$ that satisfies $\Prob{A_\tau \neq A^*} \le \delta$. Among $\delta$-PAC algorithms, the algorithm with minimum sample complexity $\E [\tau]$ is preferred.

For any set $A$, define 
\begin{align}
	\Delta(A) \triangleq \psi(A) - \psi(A^*), \textrm{ and } \Delta = \min_{A\in \A} \Delta(A).\label{eq:gap}
\end{align} 
In the above, $\Delta(A)$ denotes the gap in MSE associated with a subset $A$, while $\Delta$ denotes the smallest gap. The upper and lower bounds that we derive subsequently features these quantities.
% \subsection*{Successive Elimination For Correlated Bandits}
In the fixed confidence setting that we consider, a naive algorithm based on Algorithm 1 in \cite{even2002pac} would pull each subset equal number of times. Such an uniform sampling will be useful if all the subsets can capture the same amount of information about other subsets, i.e., when the underlying correlations and the variances are similar. However, with uneven correlations, uniform sampling does not make sense. The possible set of candidates for the most informative subset need to sampled more than the other subsets in order to reduce the probability of error in identifying the best $m$-subset, and successive elimination \cite{even2002pac} is an approach that embodies this idea.

We propose a variant of the successive elimination algorithm that is geared towards finding the best $m$- subset under the MSE objective. 
The algorithm maintains an active set, which is initialized to the set of all $m$-subsets $\A$.  In each round $t$, the algorithm pulls each active $m$-subset once, and its MSE is estimated using  \eqref{eq:mse_cleverest}. Following this, the algorithm eliminates all subsets whose confidence intervals are clearly separated from the confidence interval of the empirically best subset seen so far, i.e., the one with the least MSE estimate. The algorithm terminates when there is only  one $m$-subset left in the active set, and this event occurs with probablity one.

For deriving the confidence width $\alpha_t$ used in the successive elimination algorithm for correlated bandits (see Figure \ref{fig:flow_se} below), we start by deriving an alternative form of the bound on the MSE estimate stated in Proposition \ref{prop:adaptive-mse-conc}.
\begin{align*}
\mathbb{P}\left(|\widehat{\mse}(A)-\mse(A)| \geq \epsilon \right)
&\leq  2K \exp \left( 
    - \frac{(n-m)}{8} \min \left( 
    \frac{9\epsilon^2}{K^2}, 
    \frac{3\epsilon}{K} 
    \right) 
    \right)  \\
& \leq 2K \exp \left( - \frac{(n-m)}{8}  \left( \frac{\epsilon^2}{\frac{K^2}{9} + \frac{K\epsilon}{3}} \right) \right) \\
\end{align*}

Thus, with w.p. $1 - \delta$, we obtain
\begin{align}
& |\widehat{\mse}(A)-\mse(A)|  \leq \left(\frac{8K \log \left(\frac{2K}{\delta}\right)}{3(n-m)}  + \frac{K}{3}\sqrt{\frac{8\log \left(\frac{2K}{\delta}\right)}{(n-m) }}\right). 
\label{eq:inv-prop2-bound-refined}
\end{align}
Now, from \eqref{eq:inv-prop2-bound-refined}, we obtain the following form for the confidence width $\alpha_t$, which is used in the successive elimination algorithm for correlated bandits (see Figure \ref{fig:flow_se} below):
\begin{align}
\alpha_{t} = \frac{8\;\binom{K}{m} \log \left(\frac{2 \; \binom{K}{m}t^2}{\delta}\right)}{3  \; t} + \frac{ \binom{K}{m}}{3}
\sqrt{ \frac{8\log \left(\frac{2 \; \binom{K}{m}t^2}{\delta}\right)}{ t}}.\label{eq:alphat_appendix}
\end{align}

%The complete algorithm is given below.
%\todod{should we mention $c_{4} \geq 70$}
		\begin{figure}[H] 
	\fbox{%
		\begin{minipage}{\dimexpr\linewidth-2\fboxsep-2\fboxrule\relax}
			{\bfseries Successive elimination for correlated bandits}\ \\[0.5ex]
			{\bfseries Input}: set of all $m$-subsets $\A , |\A| = \binom{K}{m} , \delta > 0$. \\
			{\bfseries Initialization}: set of active subsets $\S = \A$. Pull each subset of $\S$ $m$-times.\\
			{\bfseries For all $t=1,2,\ldots$ , repeat}\\
			\begin{enumerate}%[(1),leftmargin=+0.2in]
			    \item Observe a sample $K$-vector from  the underlying multivariate Gaussian distribution.
				\item For each active $m$-subsets $A_t \in \S$, update the MSE estimate $\widehat{\psi}_{A_{t}}$ with the additional sample obtained above. Each active subset's MSE estimate $\widehat{\psi}_{A_{t}}$ would be formed using $t$ samples of the underlying Gaussian distribution.
				\item Remove those subsets from $\mathcal{S}$ such that \\
				$\widehat{\psi}_{A_{t}^*} - \widehat{\psi}_{A_{t}} \geq 2 \alpha_{t}$,
				where $\alpha_{t}$ is defined in \eqref{eq:alphat_appendix} and 
    $A_{t}^*$ is any active optimal subset at time $t$ with minimum MSE , i.e., $A_{t}^*  \in \argmin\limits_{A_{t} \in \mathcal{S}} \psi(A_{t}).$
				\item Continue until there is only one active $m$-subset in $\mathcal{S}$.
				%\item Continue until there is only one active $m$-subset in $\S$.
			\end{enumerate}
		\end{minipage}%
		}
	\caption{Operational flow of successive elimination for correlated bandits.}
	\label{fig:flow_se}
\end{figure}
We now present a bound on the sample complexity of the successive elimination algorithm for correlated bandits.
	\begin{theorem}[\textbf{\textit{Sample complexity bound}}]
	\label{thm:se}
	Assume \ref{ass:variance} and \ref{ass:eigenvalue} for every $A \in \A$.
	The successive elimination algorithm is $(0, \delta)$-PAC for any $\delta \in (0,1)$, and w.p. at least $1 - \delta$, it's sample complexity is bounded by
	 \begin{align*}
     &\mathcal{O} \left( \frac{1}{\Delta_{\textrm{min}}} \log \left(\frac{\binom{K}{m} \log \left(\Delta^{-1}\right)}{\delta}\right)\right),
 \end{align*}
where $\Delta = \min_{A \in \A} \Delta(A)$ denotes the smallest gap. 
	\end{theorem}
	\begin{proof}
	See Section \ref{sec:se-proof}.
	\end{proof}

A naive variant of successive elimination would obtain $|\S|$ samples of the underlying multivariate Gaussian distribution in each round, one for each active subset.
Such an algorithm would result in the sample complexity $\mathcal{O} \left( \frac{\binom{K}{m}}{\Delta} \log \left(\frac{\binom{K}{m} \log \left(\Delta^{-1}\right)}{\delta}\right)\right)$, matching the lower bound in the next section. Our variant of successive elimination in Figure \ref{fig:flow_se} avoids such a excessive sampling strategy and instead, uses a single sample of the $K$-variate Gaussian in each round to perform MSE estimation for each active subset.

%%%%%%%%%%%%%%%%%%%%%%%%%%%%%%%%%%%%%%%%%%%%%%%%%%%%%%%%%%%%%%%%%%%%%%%%%%%%%%%%%%%%%%%%%%%%%%%%%%%%%%%%%%%%%%%%%%
\section{Bandits with MSE objective: Lower Bound}
\label{sec:lb}
We consider a special case of the adaptive estimation problem, where the goal is to identify the best pair of arms, i.e.,   
\[ (i_1^*,i_2^*) \in \argmin\limits_{(i,j)\in [K]\times [K], i\ne j}  \psi(\{i,j\}).\]

We consider a class of algorithms that are $\delta$-PAC for the best pair identification problem with a bandit feedback model, where an algorithm chooses to pull an arm pair and observe a sample from the underlying bivariate distribution in each round. The algorithm can either continue sampling or choose to stop and output an arm pair. The algorithm is required to be $\delta$-PAC (i.e., find the optimal pair $(i_1^*,i_2^*)$ with probability at least $1-\delta$).

A lower bound on the sample complexity of this problem is presented below.
\begin{theorem}[\textit{\textbf{Lower bound}}]
	\label{thm:lb}
	If the number of arms $K > 7$, then for any $\delta$-PAC algorithm, there exists a bandit problem instance governed by a covariance matrix $\Sigma$ such that the sample complexity $\E_{\Sigma}[\tau_{\delta}]$ of this algorithm satisfies
	\begin{align}
		\E_{\Sigma}[\tau_{\delta}] \geq \frac{K^{2}\log (\frac{1}{2.4 \delta})}{\Delta}.\label{eq:lb}
	\end{align}
	where $\Delta$ denotes the smallest gap on the problem instance governed by $\Sigma$. 
\end{theorem}
\begin{proof}
	See Section \ref{sec:proof-lb-full}.
\end{proof}

\begin{remark}[Comparison to upper bound]
    The sample complexity bound in the Theorem \ref{thm:se} features the total number of $m$-subsets $\binom{K}{m}$ only inside the log factor. On the other hand, the lower bound that we derived above for the case of finding the best arm pair is of the form $O\left(\frac{K^2}{\Delta_{\textrm{min}}} \log\left(1/\delta\right)\right)$. 
Thus, there is an apparent contradiction between the upper and lower bounds. However, this can be resolved by observing the fact that in the successive elimination algorithm, we obtain one sample $K$-vector in each round, and estimate the MSEs of all the active subsets using this single sample along with the those from the previous rounds. On the other hand, the lower bound corresponded to a setting where the sample complexity is given by 
$\E_{\Sigma}[\tau_{\delta}] = \sum\limits_{(i,j)} \E[N_{ij}(\tau_{\delta})]$, where $N_{ij}$ is the number of times the arm pair $i,j$ is pulled before the algorithm stops.

\end{remark}

The proof strategy is to use the following  class of covariance matrices parameterized by $\rho$:
\begin{align}
	\Sigma =
	\begin{bmatrix}
		1 & \rho & \rho & \rho & \ldots & \rho\\
		\rho & 1 & \rho^2 & \rho^2 & \ldots & \rho^2 \\
		\rho & \rho^2 & 1 & \rho^3 & \ldots & \rho^3\\
		\vdots & \vdots & \vdots & \vdots & \ddots & \vdots \\
		\rho & \rho^2 & \rho^3 & \ldots & \rho^{K-1} & 1
	\end{bmatrix} \label{eq:cov_matrix}
\end{align}
Using Sylvester's criterion, it is easy to see that the matrix defined above is positive semi-definite.

For a $K$-armed Gaussian bandit instance with the underlying distribution governed by $\Sigma$ defined above, the pair $\{1,2\}$ has the least MSE.
%the following covariance matrix which is a valid candidate for the covariance matrix,also positive semi-definite.
%The MSEs corresponding to the individual pairs in the case of $K = 4$ are as follows, 
%\begin{align*}
%  \mse_{12} & = 2- \frac{2}{1- \rho^2} \left(\rho^2 + 3 \rho^4\right) \\
%  \mse_{13} = \mse_{14} & = 2 - \frac{1}{1- \rho^2} \left(2\rho^2 + 3 \rho^4 + 2\rho^5 + \rho^6 \right) \\
%  \mse_{23} = \mse_{24} & = 2 - \frac{1}{1-\rho^4} \left(2 \rho^2 + 3 \rho^4 + \rho^6 + 2 \rho^7\right) \\
%  \mse_{34} & = 2 - \frac{1}{1- \rho^6} \left(2 \rho^2 + 2 \rho^4 + 2 \rho^5 +2 \rho^7\right)
%\end{align*}

We form $(2K-4)$ transformations of the bandit instance described in \eqref{eq:cov_matrix}. The transformations are achieved by relabelling the $m^{th}$ row as either the first or second row of $\Sigma,m = 3,\ldots,K.$ Let us denote the pdf associated with the original bandit instance by $\mathcal{G}$ and $\mathcal{G}^{km}$ is the probability density function (pdf) of the transformed bandit instance obtained by relabelling the $k^{th}(k=1 \text{ or } 2)$ row and the $m^{th}$ row of $\Sigma$.

The underlying covariance matrix for the problem instance corresponding to the $m$th transformation is $\Sigma_{km}$ with $m$th row re-labelled as either  row 1 or 2.$ \text { Let } KL_{ij}^{km} \overset{\Delta}{=} KL(\nu_{i} \nu_{j}||\nu'_{i} \nu'_{j}) $
specify the KL-divergence between $\nu_{i} \nu_{j} \text{ and } \nu'_{i} \nu'_{j}$, with the latter distribution derived from $\mathcal{G}^{km}.$ 

In the proof, we first show that 
	\begin{align}
&\underset{\{w_{ij}\}}{\min} \E[\tau_{\delta}] \geq \frac{\log (\frac{1}{2.4 \delta})}{\underset{w \in \Delta_{\binom{K}{2}}}{\max}\underset{\Sigma'\in\alt(\Sigma)}{\min} \sum\limits_{i,j} w_{ij} KL\left[\Sigma_{X_{i} X_{j}}||\Sigma'_{X_{i} X_{j}}\right]},\label{eq:lb-s1}
\end{align}
where $\Delta_{\binom{K}{2}}$ is the set of probability distributions on the arm-pairs, and $\alt(\Sigma) = \{\Sigma^{1m},\Sigma^{2m}, m > 3\}$ is the set of transformed covariance matrices.
While derivation of the inequality above is a straightforward variation to the proof in the classic bandit setting (cf. \cite{kaufmann2015complexity}), the rest of the proof in our case requires significant deviations. In particular, unlike the regular bandit case, the KL-divergences in the RHS above are not univariate. Moreover, deriving an upper bound on the max-min, which is defined in the RHS above, requires arguments that are specific to our correlated bandit setting.

We would like to note that \cite{boda2019correlated} provide a lower bound for the correlated bandit problem with $m=1$. The proof of the lower bound for the case of $m=2$ is significantly different from the proof for $m=1$. In particular, it is challenging since the proof involves KL-divergences for bivariate distributions and relating these KL-divergences to the underlying gaps involves tools from optimization (see the proof sketch below), as well as significant algebraic effort
to simplify KL-divergence bounds inside the max-min in \eqref{eq:lb-s1}, and then, relating the simplified expression to the gap in MSEs of the original problem instance.
Further, unlike \cite{boda2019correlated}, the ideas in our proof for $m=2$ could be generalized to $m>2$.

%%%%%%%%%%%%%%%%%%%%%%%%%%%%%%%%%%%%%%%%%%%%%%%%%%%%%%%%%%%%%%%%%%%%%%%%%%%%%%%%
%%%%%%%%%%%%%%%%%%%%%%%%
 \section{Convergence proofs}
\label{sec:pf}
\allowdisplaybreaks

%%%%%%%%%%%%%%%%%%%%%%%%%%%%%%%%%%%%%%%%%%%%%%%%%%%%%%%%%%%%%%%%%%%%%%%%%%%%%
%%%%%%%%%%%%%%%%%%%%%%%%%%%%%%%%%%%%%%%%%%%%%%%%%%%%%%%%%%%%%%%%%%%%%%%%%%%%%
%%%%%%%%%%%%%%%%%%%%%%%%%%%%%%%%%%%%%%%%%%%%%%%%%%%%%%%%%%%%%%%%%%%%%%%%%%%%%
%%%%%%%%%%%%%%%%%%%%%%%%%%%%%%%%%%%%%%%%%%%%%%%%%%%%%%%%%%%%%%%%%%%%%%%%%%%%%
\subsection{Proof of Proposition \ref{prop:nonadaptive-mse-conc}}
\label{sec:proof-nonadaptive-conc}
For the proof of Proposition \ref{prop:nonadaptive-mse-conc}, we require tail bounds for estimates of variance, standard deviation and correlation coefficient. We state these below. The reader is referred to \cite{boda2019correlated} for the proofs.

\begin{lemma}\textbf{\textit{(Concentration of sample variance)}} 
\label{lemma:subexp-conc}
Let $X_{i}$,\  $i=1,\ldots,n,$ be independent sub-Gaussian r.v.s with common parameter $\sigma$.
Let
$\hat \sigma^2_n =  \frac{1}{n} \sum \limits_{i=1}^{n} X_{i}^{2}$. Then, we have the following bound for any $\epsilon \ge 0$:
\begin{align*}
 & P \left (  \hat\sigma^2_n >  \sigma^2 + \epsilon   \right)  \leq 
 \exp \left (  - \frac{n}{8} \min \left ( \frac{\epsilon^2}{\sigma^4}, \frac{\epsilon }{\sigma^2}\right ) \right ), \textrm { and }
  P \left ( \hat\sigma^2_n <  \sigma^2 - \epsilon   \right)  \leq 
 \exp \left (  - \frac{n}{8} \min \left ( \frac{\epsilon^2}{\sigma^4}, \frac{\epsilon }{\sigma^2}\right ) \right ).
\end{align*} 
\end{lemma}
\begin{lemma}\textbf{\textit{(Concentration of sample standard deviation)}} 
\label{lemma:sample-stddev-conc-bd}
Under conditions of Lemma \ref{lemma:subexp-conc}, letting
$\hat \sigma_n =  \sqrt{\frac{1}{n} \sum \limits_{i=1}^{n} X_{i}^{2}}$, we have
\begin{align*}
 & P \left ( \hat\sigma_n >  \sigma + \epsilon   \right)  \leq 
 \exp \left (  - \frac{n\epsilon^2}{8\sigma^4} \right ), \textrm { and }
  P \left (  \hat\sigma_n <  \sigma - \epsilon   \right) \! \leq \!
 \exp \left (  - \frac{n\epsilon^2}{8\sigma^4}\right ), \textrm{ for any } \epsilon \ge 0.
\end{align*} 
\end{lemma}
\begin{lemma}\textbf{\textit{(Concentration of sample correlation coefficient)}} 
\label{lemma:rho-conc}
For independent Gaussian rvs $X_{i}, \ i=1,\ldots,n$, with mean zero and covariance matrix $\Sigma$ as defined in \eqref{eq:gauss-covar} and with $\hat\sigma_i^2$, $\hat\rho_{ij}$ formed from $n$ samples using \eqref{eq:sample-cov}, for any $i,j=1,\ldots,K$, and for any $\epsilon \in [0,\eta]$,  we have 
\begin{align*}
 & \P \left (  \left|\hat\rho_{ij} -  \rho_{ij} \right| > \epsilon   \right) 
\le 26\exp \left ( - \frac{n}{8} \frac{1}{36(1+\eta)} \min\left( \frac{l \epsilon}{3}, \left ( \frac{l \epsilon}{3} \right)^2\right) \right), 
\end{align*} 
where $l$ is a positive constant satisfying $l \le \sigma_i^2 \le 1$, $\forall i$.
\end{lemma}

\begin{proof}(\textbf{\textit{Proposition \ref{prop:nonadaptive-mse-conc}}})
	Notice that	
	\begin{align}
	&\widehat{\mse}(A)-\mse(A) = \mathrm{Tr} \bigg(\underbrace{\left(\widehat{\Sigma_{A'A'}} - \Sigma_{A'A'}\right)}_{(I)} -\underbrace{\bigg(\widehat{\Sigma_{A'A}} \big(\widehat{\Sigma_{AA}}^{+}\big)^{-1} \widehat{\Sigma_{AA'}}  - \Sigma_{A'A} \Sigma_{AA}^{-1} \Sigma_{AA'}\bigg)}_{(II)}\bigg). \label{eq:s1_first}
	\end{align}
	The term (II) on the RHS of \eqref{eq:s1_first} can be re-written as follows:
	\begin{align*}
	(II)&=\widehat{\Sigma_{A'A}} \left(\widehat{\Sigma_{AA}}^{+}\right)^{-1} \widehat{\Sigma_{AA'}}  - \Sigma_{A'A} \left(\Sigma_{AA}\right)^{-1} \Sigma_{AA'}\\
	& =  \widehat{\Sigma_{A'A}} \left(\left(\widehat{\Sigma_{AA}}^{+}\right)^{-1} \widehat{\Sigma_{AA'}} - \Sigma_{AA}^{-1}\Sigma_{AA'}\right) + \left(\widehat{\Sigma_{A'A}} - \Sigma_{A'A}\right) \Sigma_{AA}^{-1} \Sigma_{AA'} \\
	&= \widehat{\Sigma_{A'A}} \left(\left(\widehat{\Sigma_{AA}}^{+}\right)^{-1} \left(\widehat{\Sigma_{AA'}} - \Sigma_{AA'}\right) + \left(\left(\widehat{\Sigma_{AA}}^{+}\right)^{-1} - \Sigma_{AA}^{-1}\right)\Sigma_{AA'}\right)  +\left(\widehat{\Sigma_{A'A}} - \Sigma_{A'A}\right)\Sigma_{AA}^{-1} \Sigma_{AA'} \\
	&=  \widehat{\Sigma_{A'A}} \left(\widehat{\Sigma_{AA}}^{+}\right)^{-1} \left(\widehat{\Sigma_{AA'}} - \Sigma_{AA'}\right)  + \widehat{\Sigma_{A'A}} \left(\left(\widehat{\Sigma_{AA}}^{+}\right)^{-1} - \Sigma_{AA}^{-1}\right) \Sigma_{AA'}  + \left(\widehat{\Sigma_{A'A}} - \Sigma_{A'A}\right)\Sigma_{AA}^{-1} \Sigma_{AA'}.  \stepcounter{equation}\tag{\theequation}\label{eq:1_first}
	\end{align*}
	Using the definition of the positive definite estimator $\widehat{\Sigma}^{+}$, we have
	\begin{align}
	||\widehat{\Sigma}_{AA}^{+} - \Sigma_{AA}||_{2} &\le ||\widehat{\Sigma}_{AA}^{+} - \widehat{\Sigma}_{AA}||  + ||\widehat{\Sigma}_{AA} - \Sigma_{AA}||_{2} \nonumber \\
  &\leq 
	 2\zeta + ||\widehat{\Sigma}_{AA} - \Sigma_{AA}||_{2}. 
	% &\le 2\underset{i: \hat{\lambda}_{i} \leq 0}{\max} |\hat{\lambda}_{i} - \lambda_{i}|+ ||\widehat{\Sigma}_{AA} - \Sigma_{AA}||_{2} \nonumber\\
	% &\leq 3 ||\widehat{\Sigma}_{AA} - \Sigma_{AA}||_{2}. 
 \label{eq:t11}
	\end{align}
	% In the above, we have used a corollary of the Weyl's theorem (cf. p. 161 of \cite{wainwright-notes}) to infer the following bound:
	% \[
	% \max_{\substack{i = 1,\ldots,m}} |\hat \lambda_{i}-\lambda_{i}| \leq ||\widehat{\Sigma_{AA}}-\Sigma_{AA}||_{2}.
	% \]
	Using Theorem 5.7 of \cite{rigollet2015high} in conjunction with (A1), w.p. $(1-\delta)$, we have
 	\begin{align*}
||\widehat{\Sigma_{AA}} - \Sigma_{AA}||_{2}
& \leq ||\Sigma_{AA}||_{2}  \min \bigg( \sqrt{\frac{m+\log(\frac{1}{\delta})}{n_{AA}}}, \frac{m+\log(\frac{1}{\delta})}{n_{AA}}\bigg) \\
	& \leq M_{0} \min \bigg( \sqrt{\frac{m+\log(\frac{1}{\delta})}{n_{AA}}},  \frac{m+\log(\frac{1}{\delta})}{n_{AA}}\bigg).
  \numberthis\label{eq:rigollet1}
\end{align*}
Along similar lines, we obtain
\begin{align*}
    	& ||\widehat{\Sigma_{A'A'}} - \Sigma_{A'A'}||_{2} \le M_{0} \min \bigg( \sqrt{\frac{K-m+\log(\frac{1}{\delta})}{n_{A'A'}}}, \frac{K-m+\log(\frac{1}{\delta})}{n_{A'A'}}\bigg).
 \numberthis\label{eq:rigollet}
	\end{align*}
 
  With $c= \frac1{[\lambda_{\min}(\Sigma_{AA})- \epsilon]}$,
	consider the event
	\begin{align}
	    \Tilde{\B} = & \{\sigma_{i}^2 - \epsilon \leq \hat{\sigma}_{i}^2 \leq \sigma_{i}^2 + \epsilon, i= [K], \rho_{i_{k}j} - \epsilon \leq \hat{\rho}_{i_{k}j} \leq \rho_{i_{k}j} + \epsilon,\text{ for } (k,j) \in [K], k \neq j,  ||\left(\widehat{\Sigma}_{AA}^{+}\right)^{-1}||_{2} \leq c \}. \label{eq:event-defn}
	\end{align}
	On the event $\Tilde{\B}$, w.p. $(1-\delta)$, we have 
	\begin{align*}
	 ||\left(\widehat{\Sigma_{AA}}^{+}\right)^{-1} - \Sigma_{AA}^{-1}||_{2}
 &=||\left(\widehat{\Sigma}_{AA}\right)^{-1}\big(\Sigma_{AA}-\widehat{\Sigma_{AA}}^{+}\big) \Sigma_{AA}^{-1}||_{2} \\
 & \leq ||\left(\widehat{\Sigma_{AA}}^{+}\right)^{-1}||_{2} \; ||\Sigma_{AA} -\widehat{\Sigma_{AA}}^{+}||_{2} \;||\Sigma_{AA}^{-1}||_{2}\\
	& \leq \left(\frac{ c }{M_1}\right)\left( 2 \zeta +  M_{0}\min \bigg( \sqrt{\frac{m+\log(\frac{1}{\delta})}{n_{AA}}}, \frac{m+\log(\frac{1}{\delta})}{n_{AA}}\bigg)\right)\\
 &= \frac{ 3c M_0}{M_1}\min \bigg( \sqrt{\frac{m+\log(\frac{1}{\delta})}{n_{AA}}},\frac{m+\log(\frac{1}{\delta})}{n_{AA}}\bigg),%\label{eq:aainv-bound_first}
	\end{align*}
where we used \eqref{eq:t11}, \eqref{eq:rigollet}, and substituted the value of $\zeta$ specified in the proposition statement.
	%\todod{What is $\eta$}
	
	Letting $\eta = \min\left(2K,\lambda_{min}(\Sigma_{AA})\right)$, we obtain
	$||\widehat{\Sigma_{AA'}}||_{2} \leq \sqrt{m(K-m)} \left(1+\eta\right), \text{ since } \hat{\sigma}_{j}^2 \leq \sigma_{j}^2 + \epsilon \leq 1 +\eta  \text{ on } \Tilde{\B}$ and
	$||\Sigma_{AA'}||_{2} \le \sqrt{m (K-m)} \left(1+\eta\right)$.
	Similarly,  $||\widehat{\Sigma_{A'A}}||_{2} \leq \sqrt{m (K-m)} (1+\eta) \text{ and } ||\Sigma_{A'A}||_{2} \leq \sqrt{m (K-m)} (1+\eta).$
	Thus,
 \begin{align}
	& \P\left(||\widehat{\Sigma_{AA'}} - \Sigma_{AA'}||_{2}^{2} \geq \epsilon^2,\Tilde{\B} \right) \nonumber\\
 & \leq  \P\left(||\widehat{\Sigma_{AA'}} - \Sigma_{AA'}||_{F}^{2} \geq \epsilon^2,\Tilde{\B} \right) \nonumber \\
	& = \sum\limits_{k=1}^{m} \sum\limits_{j = m+1}^{K-m} \P\bigg(|\hat{\rho}_{ji_{k}} \hat{\sigma}_{i_{k}} \hat{\sigma}_{j} - \rho_{ji_{k}} \sigma_{i_{k}} \sigma_{j}| \geq \frac{\epsilon}{\sqrt{m (K-m)}},\Tilde{\B}\bigg) \nonumber \\
	& = \sum\limits_{k=1}^{m} \sum\limits_{j = m+1}^{K-m} \P\bigg(|\hat{\rho}_{ji_{k}} \hat{\sigma}_{i_{k}} \left(\hat{\sigma}_{j} - \sigma_{j}\right)  + \hat{\rho}_{ji_{k}} \sigma_{j} \left(\hat{\sigma}_{i_{k}} - \sigma_{i_{k}}\right)  + \sigma_{i_{k}} \sigma_{i_{j}} \left(\hat{\rho}_{ji_{k}} - \rho_{ji_{k}}\right)| \geq \frac{\epsilon}{\sqrt{m (K-m)}},\Tilde{\B} \bigg) \nonumber \\
	&  \leq \sum\limits_{k=1}^{m} \sum\limits_{j = m+1}^{K-m} \Bigg(\P\left(\hat{\rho}_{ji_{k}} \hat{\sigma}_{i_{k}} \left(\hat{\sigma}_{j} - \sigma_{j}\right) \geq \frac{\epsilon}{3 \sqrt{m(K-m)}}, \Tilde{\B}\right)  + \P\left(\hat{\rho}_{ji_{k}} \sigma_{j} \left(\hat{\sigma}_{i_{k}} - \sigma_{i_{k}}\right) \geq \frac{\epsilon}{3 \sqrt{m(K-m)}}, \Tilde{\B}\right) \nonumber \\
 & \qquad\qquad\qquad\qquad + \P\left(\sigma_{i_{k}} \sigma_{j} \left(\hat{\rho}_{ji_{k}} - \rho_{ji_{k}}\right) \geq \frac{\epsilon}{3 \sqrt{m (K-m)}}, \Tilde{\B}\right)\Bigg)
	\nonumber\\
	& \leq \sum\limits_{k=1}^{m} \sum\limits_{j = m+1}^{K-m} \Bigg(\P \left(\left(\hat{\sigma}_{j} - \sigma_{j}\right) \geq \frac{\epsilon}{3 \sqrt{m (K-m)(1+\eta)}}\right) + \P\left(\left(\hat{\sigma}_{i_{k}} - \sigma_{i_{k}}\right) \geq \frac{\epsilon}{3 \sqrt{m (K-m) (1+\eta)}} \right) \Bigg.\nonumber\\
	& \qquad\qquad\qquad + \P\left(\big(\hat{\rho}_{ji_{k}} - \rho_{ji_{k}}\big) \geq \frac{\epsilon}{3  \sqrt{m (K-m)} (1+\eta)}, \Tilde{\B}\right)\Bigg) \nonumber \\
 %& \left(\text{ since } \hat{\sigma}_{j}^2 \leq \sigma_{j}^2 + \epsilon \leq 1+\eta, \text{ on } \Tilde{\B},\sigma_{i}^2,\hat{\rho}_{ij} \leq 1 \right) \nonumber \\
	& \leq \sum\limits_{k=1}^{m} \sum\limits_{j = m + 1}^{K - m} \left(\exp\left( - \frac{n_{AA'} \epsilon^2}{72 m (K-m) (1+\eta)}\right) + \exp \left(- \frac{n_{AA'} \epsilon^2}{72 m (K-m) (1 + \eta)}\right) \right.  \nonumber \\
	& \qquad \qquad \qquad \left.+ 13 \exp\left(- \frac{n_{AA'}}{8} \frac{1}{36 (1+ \eta)} \min \left(\frac{l \epsilon}{9 \sqrt{m (K-m)} (1+ \eta)},  \frac{l^2 \epsilon^2}{81 m (K-m) (1+ \eta)^2}\right)\right)\right)  \nonumber\\ 
	& \leq  \left( m (K-m) \right) \left(2 \exp\left( - \frac{n_{AA'} \epsilon^2}{72 m(K-m) (1+\eta)}\right) \right.\nonumber\\
 &\left.\qquad\qquad\qquad+ 13 \exp\left(- \frac{n_{AA'}}{8} \frac{1}{36 (1+ \eta)} \min \left(\frac{l \epsilon}{9 \sqrt{m (K-m)} (1+ \eta)},  \frac{l^2 \epsilon^2}{81 m (K-m) (1+ \eta)^2}\right)\right)\right).  \label{eq:rho_sigma_bound}
 \end{align}	
	Now, using \eqref{eq:rho_sigma_bound}, w.p. $(1 - \delta)$, we have
	\begin{align*}
	& ||\widehat{\Sigma_{AA'}} - \Sigma_{AA'}||_{2} \nonumber \\
	& \leq \bigg(\frac{1+\eta}{l}\bigg) \left( 108 \sqrt{2 m \left(K - m\right)} \right)  \min \left( \sqrt{\frac{(1 + \eta) \log \left(\frac{13 m (K - m)}{\delta}\right)}{n_{AA'}} } \left( \frac{12\sqrt{2} (1+\eta) \log \left(\frac{13 m (K - m)}{\delta}\right)} {n_{AA'}} \right) \right) \\
	& \qquad+ \sqrt{\frac{72 m (K-m) (1+\eta) \log \left(\frac{2m (K-m)}{\delta}\right) }{n_{AA'}}}. %\label{eq:rect_matrix_bound}
	\end{align*}
	Similarly, w.p. $(1 - \delta)$, we obtain
	\begin{align*}
	& ||\widehat{\Sigma_{A'A}} - \Sigma_{A'A}||_{2} \nonumber \\
	& \leq \bigg(\frac{1+\eta}{l}\bigg) \left( 108 \sqrt{2 m \left(K - m\right)} \right)  \min \left( \sqrt{\frac{(1 + \eta) \log \left(\frac{13 m (K - m)}{\delta}\right)}{n_{A'A}} },\left( \frac{12 \sqrt{2} (1+\eta) \log \left(\frac{13 m (K - m)}{\delta}\right)} {n_{A'A}} \right) \right) \\
	 & \qquad + \sqrt{\frac{72 m (K-m) (1+\eta) \log \left(\frac{2m (K-m)}{\delta}\right) }{n_{A'A}}}. %\label{eq:rect_trans_matrix_bound}
	\end{align*}
Recall term (II) from \eqref{eq:s1_first} was written in an equivalent form in  \eqref{eq:1_first}. From the bounds derived above, the term $(II)$ can be bounded on the event $\Tilde{\B}$, w.p. $(1-\delta)$, as follows:
	\begin{align*}
	(II)& \leq ||\widehat{\Sigma_{A'A}}||_{2} ||\left(\widehat{\Sigma_{AA}}^{+}\right)^{-1}||_{2} \big(||\widehat{\Sigma_{AA'}} - \Sigma_{AA'}||_{2}\big)+ ||\widehat{\Sigma_{A'A}}||_{2} ||\widehat{\Sigma_{AA'}}||_{2} \big(||\left(\widehat{\Sigma_{AA}}^{+}\right)^{-1} - \Sigma_{AA}^{-1}||_{2}\big) \\
	&\quad + ||\Sigma_{AA'}||_{2} ||\Sigma_{AA}^{-1}||_{2} \big(||\widehat{\Sigma_{A'A}} - \Sigma_{A'A}||_{2}\big)\\
	& \leq 108 \sqrt{2} \left(c m (K-m)\right) \bigg(\frac{(1 +\eta)^2}{l}\bigg)  \min \left( \sqrt{\frac{(1 +\eta) \log \left(\frac{13 m (K - m)}{\delta}\right)}{n_{AA'}}}, \right.\\
 & \qquad \qquad \qquad \qquad \qquad  \qquad \qquad \qquad \qquad \left.\left( \frac{12 \sqrt{2} (1 +\eta) \log \left(\frac{13 m (K - m)}{\delta}\right)} {n_{AA'}} \right) \right)  \\
	& \quad + c \sqrt{m (K - m)} \left(1 + \eta\right) \left( \sqrt{\frac{72 m (K-m) (1 +\eta) \log \left(\frac{2 m (K-m)}{\delta}\right) }{n_{AA'}}} \right) \\
	& \quad+ 3 \left(m (K-m)(1 +\eta)^2\right) \bigg(\frac{ c }{M_1}\bigg)  \left(2\zeta + M_{0}\min \bigg( \sqrt{\frac{m+\log(\frac{1}{\delta})}{n_{AA}}},  \frac{m+\log(\frac{1}{\delta})}{n_{AA}}\bigg)\right) \\
	& \quad+  108 \sqrt{2}  \left(\frac{m (K-m)} {M_{1}}\right) \bigg(\frac{(1 +\eta)^2}{l}\bigg)\min \left( \sqrt{\frac{(1 +\eta) \log \left(\frac{13 m (K - m)}{\delta}\right)}{n_{A'A}} }, \right. \\
 &\left. \qquad \qquad \qquad \qquad \qquad \qquad \qquad \qquad \qquad \qquad \left( \frac{12\sqrt{2} (1 +\eta) \log \left(\frac{13 m (K - m)}{\delta}\right)} {n_{A'A}} \right) \right)\\
	 & \quad +\left(\frac{\sqrt{m (K - m)}}{M_{1}}\right) \left(1 +\eta\right)  \left(\sqrt{\frac{72 m (K-m) (1 +\eta) \log \left(\frac{2 m (K-m)}{\delta}\right) }{n_{A'A}}} \right) \\
	 & \leq \left(m (K -m) (1 +\eta)^2\right) \left(C_{1} \min \left( \sqrt{\frac{(1 +\eta) \log \left(\frac{13 m (K - m)}{\delta}\right)}{n'}},  \left( \frac{12\sqrt{2} (1 +\eta) \log \left(\frac{13 m (K - m)}{\delta}\right)} {n'} \right)\right) \right. \\
 & \quad \left.+ \; C_{2} \min \bigg( \sqrt{\frac{m+\log(\frac{1}{\delta})}{n'}},\frac{m+\log(\frac{1}{\delta})}{n'}\bigg) \right) \\
 & \quad + \sqrt{m (K -m)} \left(1 +\eta \right) C_{3} \left(\sqrt{\frac{72 m (K-m) (1 +\eta) \log \left(\frac{2 m (K-m)}{\delta}\right) }{n'}}\right),
    %& \text{ where } n'=\min\left(n_{AA}, n_{A'A'}, n_{AA'}, n_{A'A}\right) \\
	%& C_{1} = \frac{108 \sqrt{2}}{l} \left(c + \frac{1}{M_{1}}\right), C_{2} = \bigg(\frac{3 c M_{0}}{M_1}\bigg), C_{3} = \left(c + \frac{1}{M_{1}}\right)
	\end{align*}
  where $n'=\min\left(n_{AA}, n_{A'A'}, n_{AA'}, n_{A'A}\right)$, $C_{1} = \frac{108 \sqrt{2}}{l} \left(c + \frac{1}{M_{1}}\right)$,  $C_{2} = \bigg(\frac{3 c M_{0}}{M_1}\bigg)$  and  
  $C_{3} = \left(c + \frac{1}{M_{1}}\right)$.
 	
	From the foregoing,
	\begin{align}
	& \mathbb{P}\left(|\widehat{\mse}(A)-\mse(A)| \geq \epsilon,\Tilde{\B}\right) \nonumber \\
	& \leq 13 \left(m (K - m) \right) \exp \left(- \frac{n'}{m (K-m)^2 (1 +\eta)^3} \min \left(\frac{\epsilon}{12\sqrt{2} C_{1}},\frac{\epsilon^2}{m (K-m)^2 (1 +\eta)^2 C_{1}^2}\right)\right) \nonumber \\
	&\quad + \exp \left(m - \frac{n'}{m (K-m)^2 (1 +\eta)^2} \min \left(\frac{\epsilon}{C_{2}}, \frac{\epsilon^2}{m (K-m)^2 (1 +\eta)^2 C_{2}^2}\right)\right) \nonumber \\
	& \quad+ \left(2 m (K-m) \right)  \exp \left( - \frac{n'}{72 \left(m (K-m)^2\right)^2 (1 +\eta)^3} \frac{\epsilon^2}{C_{3}^2}\right) \nonumber \\
	& \quad+ \exp \left(\left(K -m\right) - n' \min \left(\frac{\epsilon}{(K-m) M_{0}}, \frac{\epsilon^2}{(K-m)^2 M_{0}^2}\right)\right)  + \frac{ 6m (K-m)(1 +\eta)^2 c \zeta}{M_1}.\stepcounter{equation}\tag{\theequation}\label{eq:mb_first}
	\end{align}
	\par
	Let \;$ \lambda_{\min}(\Sigma_{AA})$ and $ \lambda_{\min} \left(\widehat{\Sigma_{AA}}^{+}\right) $ be the smallest eigenvalues of $\Sigma_{AA}$ and $\widehat{\Sigma_{AA}}^{+}$ respectively. Then, for $0 < \epsilon < \eta$, we have
	\begin{align*}
	 \mathbb{P}\big(\lambda_{\min}\left(\widehat{\Sigma_{AA}}^{+}\right) \leq \lambda_{min}(\Sigma_{AA}) - \epsilon\big) &=  \;  \mathbb{P}\big(\lambda_{\min}\left(\widehat{\Sigma_{AA}}^{+}\right) \leq 1/c\big) 	= \; \mathbb{P}\big(1/\lambda_{\min}\left(\widehat{\Sigma_{AA}}^{+}\right) \geq c) \\
	&= \;  \mathbb{P}\big(||\left(\widehat{\Sigma_{AA}}^{+}\right)^{-1}||_{2} \geq c).
	\end{align*}
	Using a corollary of the Weyl's theorem (cf. p. 161 of \cite{wainwright-notes}), we obtain
	\begin{align*}
	 \mathbb{P}\big(\lambda_{\min}\left(\widehat{\Sigma_{AA}}^{+}\right)-\lambda_{\min}(\Sigma_{AA}) \geq \epsilon \big) 
 = \; & \mathbb{P}\big(||\left(\widehat{\Sigma_{AA}}^{+}\right)^{-1}||_{2}\geq c\big) 
 \leq \;  \mathbb{P}\big(||\widehat{\Sigma_{AA}}^{+} -\Sigma_{AA}||_{2} \geq \epsilon\big).
	%\label{eq:6}
	\end{align*}
	From \eqref{eq:rigollet} and \eqref{eq:t11}, w.p. at least $(1-\delta)$, we have
	\begin{align*}
	||\widehat{\Sigma_{AA}}^{+} - \Sigma_{AA}||_{2} & \leq 2\zeta  
+ M_{0} \min \bigg( \sqrt{\frac{m+\log(\frac{1}{\delta})}{n_{AA}}},\frac{m+\log(\frac{1}{\delta})}{n_{AA}}\bigg)\\
 & = 3 \; M_{0} \min \bigg( \sqrt{\frac{m+\log(\frac{1}{\delta})}{n_{AA}}}, \frac{m+\log(\frac{1}{\delta})}{n_{AA}}\bigg),
	\end{align*}
 where the final equality is obtained by substituting the value of $\zeta$ specified in the proposition statement.
	Hence,
	\begin{align}
	\mathbb{P}\big(\Tilde{\B'}\big) \leq \exp \bigg(m-n_{AA}\min\bigg(\frac{\epsilon}{ 3 M_{0}},\frac{\epsilon^2}{ 9 M_{0}^2}\bigg)\bigg).\label{eq:complement-bound_first}
	\end{align}
	Combining \eqref{eq:mb_first} and \eqref{eq:complement-bound_first}, we obtain
	\begin{align*}
	& \mathbb{P}\left(|\widehat{\mse}(A)-\mse(A)| \geq \epsilon \right)  \\
 \leq \; & \mathbb{P}\left(|\widehat{\mse}(A)-\mse(A)| \geq \epsilon, \Tilde{\B} \right) + \mathbb{P}(\Tilde{\B'})\\
	 \leq & \; 13 \left(m (K - m) \right) \exp \left(- \frac{n'}{m (K-m)^2 (1 +\eta)^3} \min \left(\frac{\epsilon}{12\sqrt{2} C_{1}},  \frac{\epsilon^2}{m (K-m)^2 (1 +\eta)^2 C_{1}^2}\right)\right) \\
	& +  \exp \left(m - \frac{n'}{m (K-m)^2 (1 +\eta)^2} \min \left(\frac{\epsilon}{C_{2}},\frac{\epsilon^2}{m (K-m)^2 (1 +\eta)^2 C_{2}^2}\right)\right)\\
	& + \left(2 m (K-m)\right)  \exp \left( - \frac{n'}{72 \left(m (K-m)^2\right)^2 (1 +\eta)^3} \frac{\epsilon^2}{C_{3}^2}\right) \\
	& + \exp \left(\left(K -m\right) - n' \min \left(\frac{\epsilon}{(K - m) M_{0}},\frac{\epsilon^2}{(K-m)^2 M_{0}^2}\right)\right)  + \exp \left(m - n_{AA} \min \left(\frac{\epsilon}{ 3  M_{0}},\frac{\epsilon^2}{9  M_{0}^2}\right)\right) \\
	\leq & \left( 13 m \left(K - m\right)  + \exp \left(m + K\right)\right)\exp \left(- \frac{n'}{m (K-m)^2 (1 +\eta)^3} \min \left(\frac{\epsilon}{12\sqrt{2} G_{0}},\frac{\epsilon^2}{ G_{0}^2}\right)\right) \\
	& \quad  + \left(2 m (K-m)\right) \exp \left( - \frac{n'}{72 \left(m (K-m)^2\right)^2 (1 +\eta)^3} \frac{\epsilon^2}{C_{3}^2}\right),
	%& \text{ where } G_{0} = \max \left( \sqrt{m (K-m) } (1 +\eta) \; C_{1}, \sqrt{m (K-m) } (1 +\eta) \; C_{2}, 3 \; M_{0}\right)
	\end{align*}
 where  $G_{0} = \max  \bigg( \sqrt{m (K-m)^2 } (1 +\eta) C_{1},  \sqrt{m (K-m)^2 } (1 +\eta) \; C_{2},  3 (K-m) \; M_{0}\bigg)$.
 
 Hence proved.
\end{proof}

%%%%%%%%%%%%%%%%%%%%%%%%%%%%%%%%%%%%%%%%%%%%%%%%%%%%%%%%%%%%%%%%%%%%%%%%%%%%%
%%%%%%%%%%%%%%%%%%%%%%%%%%%%%%%%%%%%%%%%%%%%%%%%%%%%%%%%%%%%%%%%%%%%%%%%%%%%%
%%%%%%%%%%%%%%%%%%%%%%%%%%%%%%%%%%%%%%%%%%%%%%%%%%%%%%%%%%%%%%%%%%%%%%%%%%%%%
%%%%%%%%%%%%%%%%%%%%%%%%%%%%%%%%%%%%%%%%%%%%%%%%%%%%%%%%%%%%%%%%%%%%%%%%%%%%%

%%%%%%%%%%%%%%%%%%%%%%%%%%%%%%%%%%%%%%%%%%%%%%%%%%%%%%%%%%%%%%%%%%%%%%%%%%%%%%%%%%%%%%%%%%%%%%%%%%%%%%%%%%%%%%%%%%%%%%%%%%%%%%%%%%%%%%%%%%%%%%%%%%%%%%%%%%%%%%%%%%%%%%%5

\subsection{Proof of Proposition \ref{prop:adaptive-mse-conc}}

\label{sec:proof-adaptive-conc}

\begin{proof}
    A random variable \( \xi \) is called sub-Gaussian if its distribution is dominated by that of a normal random variable. Formally, this can be expressed as requiring that \( \mathbb{E}\exp(\xi^2 / \kappa^2) \leq 2 \) for some \( \kappa > 0 \). The infimum of such \( \kappa \) is known as the sub-Gaussian norm or \( \psi_2 \)-norm of \( \xi \), denoted as \( \|\xi\|_{\psi_2} \).

To establish the concentration bound in Proposition \ref{prop:adaptive-mse-conc}, we employ the Hanson-Wright inequality \cite{rudelson2013hansonwrightinequalitysubgaussianconcentration}, which is stated below.

\textbf{Theorem (Hanson-Wright inequality).}  
Let \( X = (X_1, \dots, X_n) \in \mathbb{R}^n \) be a random vector with independent components \( X_i \) such that \( \mathbb{E}[X_i] = 0 \) and \( \|X_i\|_{\psi_2} \leq c_2 \). Let \( A \) be an \( n \times n \) matrix. Then, for every \( \epsilon \geq 0 \),
\begin{equation}
    \mathbb{P} \left( \left| X^\top A X - \mathbb{E}[X^\top A X] \right| > \epsilon \right) 
    \leq 2 \exp \left[ 
    - \frac{1}{8} \min \left( 
    \frac{\epsilon^2}{c_2^4 \|A\|_{\text{F}}^2}, 
    \frac{\epsilon}{c_2^2 \|A\|} 
    \right) 
    \right],
    \label{eq:hanson_wright}
\end{equation}
where \( \|A\|_{\text{F}} \) is the Frobenius norm of \( A \), \( \|A\| \) is the operator norm of \( A \).

Applying this to our estimator, we obtain:
\begin{align}
    \mathbb{P} \left( \left| \widetilde{\psi}(A) - \psi(A) \right| > \epsilon \right) &\leq \sum_{j=1}^K \mathbb{P} \left( \left| \frac{X_j^\top (I - P_A) X_j}{n - m} - \gamma_j^2 \right| > \frac{\epsilon}{K} \right) \nonumber \\ 
    &\leq \sum_{j=1}^K 2 \exp \left[ 
    - \frac{1}{8} \min \left( 
    \frac{(n-m)^2\epsilon^2}{K^2 c_2^4 \|I - P_A\|_{\text{F}}^2}, 
    \frac{(n-m)\epsilon}{K c_2^2 \|I - P_A\|} 
    \right) 
    \right] \nonumber \\
    &\leq 2K \exp \left[ 
    - \frac{1}{8} \min \left( 
    \frac{(n-m)\epsilon^2}{K^2 c_2^4}, 
    \frac{(n-m)\epsilon}{K c_2^2} 
    \right) 
    \right].
    \label{eq:concentration_bound}
\end{align}

In the last step, we used the fact that the Frobenius norm of \( (I_{n} - P_A) \) is $\sqrt{n - m}$, since \( (I_{n} - P_A) \) is idempotent, with eigenvalues \( 0 \) or \( 1 \), and the sum of eigenvalues equal to the rank \( n - m \). Similarly, the operator norm of \( (I_{n} - P_A) \) is \( 1 \). The matrix $A = \frac{I_n - P_A}{n-m}$ has the Frobenius norm $= \frac{1}{\sqrt{n-m}}$ and the operator norm $= \frac{1}{n-m}$. 

We can further simplify the bound by using the fact that we are dealing with Gaussian random variables, which in turn would lead to specifying $c_2$ explicitly. 
Recall \( X_{j} = (X_{j1}, \dots, X_{jn}) \sim \mathcal{N}(0, \sigma_j^2I_{n}) \). To find the sub-Gaussian norm \( c_{2} \), we start with the condition:
\[
\mathbb{E} \left[\exp\left(\frac{X_{ji}^2}{c_{2}^2}\right) \right] \leq 2 \qquad \forall i\in [1,...,n].
\]
Using a characterization of sub-Gaussian random variables (cf. Theorem 3 of \cite{rivasplata2012subgaussian}), we can infer that $c_{2} = \frac{1}{\sqrt{3}} \sigma_{j}$. Using  \ref{ass:variance}, we obtain $c_{2} \leq \frac{1}{\sqrt{3}}$.

\end{proof}

%%%%%%%%%%%%%%%%%%%%%%%%%%%%%%%%%%%%
%%%%%%%%%%%%%%%%%%%%%%%%%%%%%%%%%%%%%%%%%%%%%%%%%%%%%%%%%%%%%%%%%%%%%%%%%%%%%
%%%%%%%%%%%%%%%%%%%%%%%%%%%%%%%%%%%%%%%%%%%

%%%%%%%%%%%%%%%%%%%%%%%%%%%%%%%%%%%%%%%%%%%%%%%%%%%%%%%%%%%%%%%%%%%%%%%%%%%%%
%%%%%%%%%%%%%%%%%%%%%%%%%%%%%%%%%%%%%%%%%%%%%%%%%%%%%%%%%%%%%%%%%%%%%%%%%%%%%
%%%%%%%%%%%%%%%%%%%%%%%%%%%%%%%%%%%%%%%%%%%%%%%%%%%%%%%%%%%%%%%%%%%%%%%%%%%%%
%%%%%%%%%%%%%%%%%%%%%%%%%%%%%%%%%%%%%%%%%%%%%%%%%%%%%%%%%%%%%%%%%%%%%%%%%%%%%

\subsection{Proof of Theorem \ref{thm:lb}}
\label{sec:proof-lb-full}
\begin{proof}
The basis of all the calculations is an established result for the KL-divergence between multivariate Gaussian distributions stated below (\cite{pardo2005statistical}).

\begin{lemma}
	\label{lemm:KLdivg_gaussian}
	Let $\mathcal{N}_{0},\mathcal{N}_{1}$ be two k-dimensional normal distribution with zero-mean 
	and covariance matrix $A_{0},A_{1}$, respectively,
	\begin{align*}
	& \hspace{-2em}KL\left(\mathcal{N}_{0}||\mathcal{N}_{1}\right) = \frac{1}{2} \left[Tr(A_{1}^{-1}A_{0}) - k + \ln\left(\frac{\det(A_{1})}{\det(A_{0})}\right)\right]
	\end{align*}
\end{lemma}
Using this standard result, we bound KL-divergence between original and transformed problem instances below.

\begin{align*}
    & \hspace{-2em}\textbf{Case } \mathbf{ m < j < k}:\\
    & \hspace{-2em} \textrm{When the } i\textrm{th}\;(i \in \{1,2\}) \textrm{ and the } m\textrm{th row of } \Sigma \textrm{ are relabeled, the matrices } A_{0} \textrm{ and } A_{1} \textrm{ are }
    \begin{bmatrix}
		1 & \rho^i \\
		\rho^i & 1
	\end{bmatrix} \text{ and } \\
	& \hspace{-2em} \begin{bmatrix}
		1 & \rho^m \\
		\rho^m & 1
	\end{bmatrix}. \\
 & \hspace{-2em} \textrm{ Thus, }\\
 & \qquad \qquad KL_{1j}^{1m} =  \frac{1}{2} \left( 2 \frac{(1 - \rho^{m+1})}{(1 - \rho^{2m})} - 2 + \ln \left(\frac{1 - \rho^{2m}}{1 - \rho^2}\right)\right)  \leq \frac{\rho^2}{2} \left(2 \frac{(\rho^{2m-2} - \rho^{m-1})}{(1 - \rho^2m}) + \frac{1 - \rho^{2m-2}}{1 - \rho^2}\right) \\
		& \qquad \qquad \qquad \qquad \qquad \qquad \qquad \qquad \qquad \qquad \qquad \qquad \quad \leq \frac{\rho^2}{2} \left(\frac{1 - \rho^{2(m-1)}}{1 - \rho^2}\right). \\
  & \textrm{Similarly},\\
		& \qquad \qquad KL_{2j}^{2m}  \leq \frac{\rho^4}{2}  \left(\frac{1 - \rho^{2(m-2)}}{1 - \rho^4}\right),  KL_{mj}^{1m}
		\leq \frac{\rho^2}{2} \left(\frac{1 - \rho^{m-1}}{1 - \rho^2}\right), \textrm{ and } KL_{mj}^{2m}  \leq \frac{\rho^4}{2} \left(\frac{1 - \rho^{m-2}}{1 - \rho^4}\right).
\end{align*}
	\begin{align*}
        & \hspace{-8em}\textbf{Case } \mathbf{ 1 < j < m:} \\
		& \hspace{-8em}KL_{1j}^{1m}  \leq \frac{\rho^2}{2}  \left(\frac{1 - \rho^{2(j-1)}}{1 - \rho^2}\right), KL_{2j}^{2m}  \leq \frac{\rho^4}{2}  \left(\frac{1 - \rho^{2(j-2)}}{1 - \rho^4}\right), KL_{mj}^{1m} \leq \frac{\rho^2}{2} \left(\frac{1 - \rho^{j-1}}{1 - \rho^2}\right), \textrm{ and } \\ 
  & \hspace{-8em} KL_{mj}^{2m} \leq \frac{\rho^4}{2} \left(\frac{1 - \rho^{j-2}}{1 - \rho^4}\right)\\
  & \hspace{-8em} KL_{im}^{im} = 0 \textrm{ for } i \in \{1, 2\}, KL_{ij}^{km} = 0 \textrm{ for } k \in \{1, 2\}, i \notin \{1, 2, m\} \cap j \notin \{1, 2, m\} .
	\end{align*}
	%$KL_{im}^{im} = 0$ for $i \in \{1, 2\}$, $KL_{ij}^{km} = 0$ for $k \in \{1, 2\}, i \notin \{1, 2, m\} \cap j \notin \{1, 2, m\} .$
	
%	We shall use $\Sigma'$ to index this set, while deriving the lower bound on $\mathbf{\E}_{\mathbf{\Sigma}}[\mathbf{\tau_{\delta}}]$ below.
Notice that $\E_{\Sigma}[\tau_{\delta}] = \sum\limits_{(i,j)} \E[N_{ij}(\tau_{\delta})]$.
	For any $\Sigma' \in \alt(\Sigma)$, from Lemma 1 and Remark 2 of \cite{kaufmann2015complexity}, we have
	\[
	\sum\limits_{(i,j)} \E[N_{ij}(\tau_{\delta})]  KL\left(\Sigma_{X_{i} X_{j}}||\Sigma'_{X_{i} X_{j}}\right) \geq \log\left (\frac{1}{2.4 \delta}\right).
	\]
 Consider the following optimization problem, with $\alpha_{ij} \triangleq \E[N_{ij}(\tau_{\delta})]$:
	
	\[\underset{\{\alpha_{ij}\}}{\min} \sum\limits_{(i,j)} \alpha_{ij} \quad\textrm{ subject to } \quad
	\E[\tau_{\delta}] \sum\limits_{(i,j)} \frac{\alpha_{ij}}{\E[\tau_{\delta}]} KL\left(\Sigma_{X_{i} X_{j}}||\Sigma'_{X_{i} X_{j}}\right) \geq 
	\log \left(\frac{1}{2.4 \delta}\right), \forall \Sigma' \in \alt(\Sigma).
	\]
 Letting $w_{ij} \triangleq \frac{\alpha_{ij}}{\E[\tau_{\delta}]},$ the problem defined above is equivalent to the following:
		\[\underset{\{w_{ij}\}}{\min} \E[\tau_{\delta}] \quad\textrm{ subject to } \quad
	\E[\tau_{\delta}]\; \underset{\Sigma'\in\alt(\Sigma)}{\min} \sum\limits_{(i,j)} w_{ij} KL\left(\Sigma_{X_{i} X_{j}}||\Sigma'_{X_{i} X_{j}}\right)\!\geq \log \left[\frac{1}{2.4 \delta}\right].
	\]
 Hence,
	\[
	\underset{\{w_{ij}\}}{\min} \E[\tau_{\delta}] \geq  \frac{\log (\frac{1}{2.4 \delta})}{\underset{w \in \Delta_{\binom{K}{2}}}{\max}\underset{\Sigma'\in\alt(\Sigma)}{\min} \sum\limits_{(i,j)} w_{ij} KL\left(\Sigma_{X_{i} X_{j}}||\Sigma'_{X_{i} X_{j}}\right)}.
	\]
	Next, we derive an upper bound on the max-min in the denominator above.\\
	Let $f(w,\Sigma') \triangleq \sum\limits_{(i,j)} w_{ij} KL\left(\Sigma_{X_{i} X_{j}}||\Sigma'_{X_{i} X_{j}}\right).$
    Then, we have
    \begin{align*}
		 f(w,\Sigma^{13})  &= w_{12} KL^{13}_{12} + w_{14} KL^{13}_{14} + \ldots + w_{1K} KL^{13}_{1K}  + w_{23} KL^{13}_{23} + w_{34} KL^{13}_{34} + \ldots + w_{3K} KL^{13}_{3K} \\
		& \leq \frac{\rho^2}{2 (1 - \rho^2)} \left(w_{12} (1 - \rho^2)  + (w_{14} + \ldots + w_{1K})  (1 - \rho^4) + w_{23} (1 - \rho) + (w_{34} + \ldots + w_{3K}) \right. \\
  &\left. \qquad \qquad\qquad\times(1 - \rho^2) \right), \\
%		& f(w, \Sigma^{14}) = w_{12} KL^{14}_{12} + w_{13} KL^{14}_{13} + \ldots + w_{1K} KL^{14}_{1K} \\
%		& + w_{24} KL^{14}_{24} + w_{34} KL^{14}_{34} + \ldots + w_{4K} KL^{14}_{4K} \\
%		& \leq \frac{\rho^2}{2 (1 - \rho^2)} \left(w_{12} (1 - \rho^2) + w_{13} (1 - \rho^4) \right.\\
%		& \left. + (w_{15} + \ldots + w_{1K}) (1 - \rho^6) + w_{24} (1 - \rho) \right.\\
%		& \left.+ w_{34} (1 - \rho^2)  + (w_{45} + \ldots + w_{4K}) (1 - \rho^3) \right), \textrm{and, so on }\\
		&\mathrel{\makebox[\widthof{=}]{\vdots}} \\
%		& \textrm{ and }\\
		 f(w, \Sigma^{1m}) 
  &= \sum\limits_{\substack{j=2 \\ j\neq m}}^K \left(w_{1j} KL^{1m}_{1j} + w_{mj} KL^{1m}_{mj} \right) \\
		& \leq \frac{\rho^2}{2 (1 - \rho^2)} \left(\sum\limits_{2 \leq j < m} w_{1j} (1 - \rho^{2(j-1)})  + \sum\limits_{m < j \leq K} w_{1j} (1 - \rho^{2(m-1)}) + \sum\limits_{2 \leq j < m} w_{mj} (1 - \rho^{j - 1}) \right. \\
  &\left. \qquad \qquad \qquad \qquad + \sum\limits_{m < j \leq K} w_{mj} (1 - \rho^{m-1}) \right).
	\end{align*}
   Along similar lines, 
   \begin{align*}
		 f(w,\Sigma^{23}) 
   &\leq \frac{\rho^4}{2 (1 - \rho^4)} \left((1 - \rho^2) (w_{24} +\ldots + w_{2K}) + (1 - \rho) (w_{34} + \ldots + w_{3K})\right), \\
%		& f(w,\Sigma^{24})\leq \frac{\rho^4}{2 (1 - \rho^4)} \times\\
%		&\left((1 - \rho^2) w_{23} + (1 - \rho^4) (w_{25} +\ldots + w_{\eta}) \right.\\
%		& \left. + w_{34} (1 - \rho) + (1 - \rho^2) (w_{45} + \ldots + w_{4K})\right),
		&\mathrel{\makebox[\widthof{=}]{\vdots}} \\
%         \textrm{and, so on  }\\
		 f(w,\Sigma^{2m}) &= \sum\limits_{\substack{j=3 \\ j\neq m}}^K \left(w_{2j} KL^{2m}_{2j} + w_{mj} KL^{2m}_{mj}\right)\\ 
		&  \leq \frac{\rho^4}{2 (1 - \rho^4)} \left( \sum\limits_{3 \leq j < m} w_{2j} (1 - \rho^{2(j-2)}) \sum\limits_{m < j \leq K} w_{2j} (1 - \rho^{2(m-2)})  + \sum\limits_{3 \leq j < m} w_{mj} (1 - \rho^{j-2}) \right. \\
  &\left. \qquad \qquad \qquad \quad + \sum\limits_{m < j \leq K} w_{mj} (1 - \rho^{m-2})\right).
	\end{align*}
     Notice that  $\underset{w}{\min} f(w,\Sigma^{1m}) \geq \underset{w}{\max} f(w,\Sigma^{2m}).$
    This inequality holds because
	$\rho^2 (1 - \rho^{2(m-2)}) \geq (1 + \rho^2) (1 - \rho)$  for $m \geq 3$, and $\rho \in [-1,1]$.
	
	Now,
	\begin{align*}
		& \underset{w}{\max} \min \left(f(w,\Sigma^{23}),f(w,\Sigma^{24}),\ldots,f(w,\Sigma^{2K})\right)\\ 
		= & \max_w \underset{\begin{subarray}{c}
				\alpha_{1},\alpha_{2},\ldots,\alpha_{K} \geq 0 \\
				\alpha_{1}+\alpha_{2}+\ldots+\alpha_{K} = 1
		\end{subarray}} {\min} \left\{\alpha_{1} f(w,\Sigma^{23}) + \alpha_{2} f(w,\Sigma^{24})  + \ldots + \alpha_{K} f(w,\Sigma^{2K})\right\} \\
		\leq & \underset{\begin{subarray}{c}
				\alpha_{1},\alpha_{2},\ldots,\alpha_{K} \geq 0 \\
				\alpha_{1}+\alpha_{2}+\ldots+\alpha_{K} = 1
		\end{subarray}} {\min} \max_w \; \left\{\alpha_{1} f(w,\Sigma^{23})   + \alpha_{2} f(w,\Sigma^{24}) + \ldots  + \alpha_{K} f(w,\Sigma^{2K})\right\} \\
		 \leq &  \;\;\underset{w}{\max} f(w,\Sigma^{23}) \qquad \left(\text{Choosing } \alpha_{1} = 1, \alpha_{i} = 0, i=2,\ldots,K\right) \\
		= &  \;\max_w \bigg\{\frac{\rho^4}{2 (1 - \rho^4)} \left((1 - \rho^2) (w_{24} +\ldots + w_{2K})+ \; (1 - \rho) (w_{34} + \ldots + w_{3K})\right)\bigg\} \\
		 = & \;  \frac{\rho^4}{2 (1 + \rho^2)},
	\end{align*}
where the final inequality holds for  any $\rho \in [0,1]$, with the following optimal weights: $\sum\limits_{j = 4}^{K} w_{2j} = 1$, and $\sum\limits_{j = 4}^{K} w_{3j} = 0.$

	Notice that the smallest gap $\Delta$ for the bandit instance governed by $\Sigma$ is given by  
	\begin{align*}
		\Delta & = \mse(\{2,3\}) - \mse(\{1,2\}) \\
		& = \frac{1}{(1 - \rho^4)} \left[(-2K+6)\rho^7 - \rho^6 + (K-1)\rho^4 + (K-4) \rho^2 \right].
	\end{align*}

and we wish to show that $\frac{\rho^4}{2 (1  + \rho^2)} \leq \frac{\Delta}{K^{2}}$.  Define the bi-variate function $$g(\rho,K) = \frac{\rho^4}{4 (1  + \rho^2)} - \frac{1}{K^2(1 - \rho^4)} \left[(-2K+6)\rho^7 - \rho^6 + (K-1)\rho^4 + (K-4) \rho^2 \right].$$ 
If $\rho = 0$ (uncorrelated arms) then $g(\rho,K) = 0$ is true $\forall K$. Otherwise, upon setting $g(\rho,K) = 0$, we get a quadratic in $K$ as: $$(\rho^4 - \rho^8)K^2 +(4\rho^9 + 4\rho^7 - 2\rho^6 - 4\rho^4 - 2\rho^2)K - (12 \rho^9 - 2\rho^8 - 4\rho^6 + 12\rho^7 - 10\rho^4 - 8\rho^2) = 0.$$ 
The two roots of the quadratic equation above are as  follows: $$\frac{(2\rho^4 + 2\rho^3+2\rho^2+\rho+1) \pm \sqrt{(4\rho^8 - 4\rho^7 - 10\rho^6 - 8\rho^5 - 6\rho^4 - 8\rho^3 - 3\rho^2 + 2\rho + 1)}}{\rho^3 + \rho^2}.$$ 
Let $K_{1}(\rho)$ and $K_{2}(\rho)$ denote the two roots, with $K_{1} \leq K_{2}$. Our desired result $g(\rho,K) \leq 0$ holds between $K_{1}$ to $K_{2}$, provided that the discriminant $D \geq 0$. Note that $D = 0$ has real roots $\rho_{1} \approx -0.5897, \rho_{2} \approx -0.5292, \rho_{3} \approx 0.4572$ in the interval $[-1,1]$, while the other roots either fall outside this interval or are imaginary. 

For the remainder of the proof, we restrict ourselves to the region $\rho \in (-1,-0.60]$. It is easy to see that in this region $K_{1}(\rho)$ is a monotonically increasing function with $\lim_{\rho \to -1^+} K_{1}(\rho) = 3$ and $\lim_{\rho \to -0.60^{-}} K_{1}(\rho) < 7$. Thus the smaller root is always less than $7$. The difference between the 2 roots $$d(\rho) = K_{2}-K_{1} = \frac{2\sqrt{4\rho^8 - 4\rho^7 - 10\rho^6 - 8\rho^5 - 6\rho^4 - 8\rho^3 - 3\rho^2 + 2\rho + 1}}{\rho^3 + \rho^2}$$ can be seen to be a monotonically increasing function since its first derivative: $$d'(\rho) = \frac{4(2\rho^9 + 3\rho^8 - 3\rho^7 - 3\rho^6 + \rho^5 + 6\rho^4 + 5\rho^3 - \rho^2 - 3\rho -1)}{\rho^3(\rho+1)^2\sqrt{(4\rho^8 - 4\rho^7 - 10\rho^6 - 8\rho^5 - 6\rho^4 - 8\rho^3 - 3\rho^2 + 2\rho + 1)}} > 0 \quad \text{for} \quad \rho \in (-1,-0.6].$$ 
Also, $\lim_{\rho \to -1^+} d(\rho) = \infty$ and $\lim_{\rho \to -0.60^{-}} d(\rho) < 1$. Thus, we have managed to show that for every value of $K \geq 7$, there exists a unique $\rho \in (-1,-0.60)$ such that there exists a bandit problem instance governed by a covariance matrix $\Sigma$ such that $\frac{\rho^4}{2 (1  + \rho^2)} \leq \frac{\Delta}{K^{2}}$. Moreover for a given $K$, this unique value of $\rho$ satisfies $K_{2}(\rho) = K$ or: $$\frac{(2\rho^4 + 2\rho^3+2\rho^2+\rho+1) + \sqrt{(4\rho^8 - 4\rho^7 - 10\rho^6 - 8\rho^5 - 6\rho^4 - 8\rho^3 - 3\rho^2 + 2\rho + 1)}}{\rho^3 + \rho^2} = K.$$
Hence proved.

\end{proof}

%%%%%%%%%%%%%%%%%%%%%%%%%%%%%%%%%%%%%%%%%%%%%%%%%%%%%%%%%%%%%%%%%%%%%%%%%%%%%
%%%%%%%%%%%%%%%%%%%%%%%%%%%%%%%%%%%%%%%%%%%%%%%%%%%%%%%%%%%%%%%%%%%%%%%%%%%%%
%%%%%%%%%%%%%%%%%%%%%%%%%%%%%%%%%%%%%%%%%%%%%%%%%%%%%%%%%%%%%%%%%%%%%%%%%%%%%
%%%%%%%%%%%%%%%%%%%%%%%%%%%%%%%%%%%%%%%%%%%%%%%%%%%%%%%%%%%%%%%%%%%%%%%%%%%%%
%%%%%%%%%%%%%%%%%%%%%%%%%%%%%%%%%%
%%%%%%%%%%%%%%%%%%%%%%%%%%%%%%%%%%%%%%%%%%%%%%%%%%%%%%%%%%%%%%%%%%%%%%%%%%%%%
%%%%%%%%%%%%%%%%%%%%%%%%%%%%%%%%%%%%%%%%%%%%%%%%%%%%%%%%%%%%%%%%%%%%%%%%%%%%%
\subsection{Proof of Theorem \ref{thm:se}}
\label{sec:se-proof}
\begin{proof}
	\begin{align*}
	& \hspace{-4em}\text{ Define the event } E  = \left\{|\widehat{\mse}_{A_{t}} - \mse_{A}| < \alpha_{t}, \; \forall t = 1, 2, \ldots \textrm{ and } \forall A_{t} \in \A \right\}. 
	\end{align*}
	We establish below that $\P \left( E'\right) \leq \delta$.
	\begin{align*}
	\P \left( E' \right) & = \P \left(\sum\limits_{t = 1}^{\infty} \sum \limits_{\A} \left( |\widehat{\mse}_{A_{t}} - \mse_{A}| \geq \alpha_{t} \right)\right) \\
	& \leq \sum\limits_{t = 1}^{\infty} \sum \limits_{\A} \P \left(\left( |\widehat{\mse}_{A_{t}} - \mse_{A}| \geq \alpha_{t} \right)\right) \\
	& \leq \sum\limits_{t = 1}^{\infty} \sum \limits_{\A}  2 K  \exp \left(- \frac{(n-m)  \alpha_{t}^2}{8\left(\frac{K^2}{9} + \frac{K\alpha_{t}}{3}\right) }\right) \\
	& \leq \sum\limits_{t = 1}^{\infty} \sum \limits_{\A} 2 K  \exp \left(- \frac{(n-m) \; \log \left(\frac{2 \; \binom{K}{m} t^2}{\delta}\right)}{t}\right)\\
	&  \leq \sum\limits_{t = 1}^{\infty} \sum \limits_{\A} 2K  \exp \left(- \log \left(\frac{2 \; \binom{K}{m} t^2}{\delta}\right) \right) \\
	& \qquad \textrm{ since } (n-m) \geq t\\
	& \leq \sum\limits_{t = 1}^{\infty} \sum \limits_{\A} 2K  \left(\frac{\delta}{2 \; \binom{K}{m} t^2}\right) \\
	& \leq K \sum\limits_{t = 1}^{\infty} \left(\frac{1}{t^2}\right) \sum \limits_{A} \left(\frac{\delta}{ \; \binom{K}{m} }\right) \\
	& \leq K  \sum \limits_{A} \left(\frac{\delta}{ \; \binom{K}{m} }\right) \leq \delta.
	\end{align*}
	%For the above to hold true, X must be $\mathcal{O} \left(\binom{K}{m} (Km^2) \right)$.
	Now, we show that with probability $1 - \delta$, the best subset can never be eliminated. The best subset $A^*$ gets eliminated if at some time $t$, for some suboptimal subset $A_{t}$, the following condition holds 
	\begin{align}
	\widehat{\mse}_{A_{t}} + \alpha_{t} < \widehat{\mse}_{A^*} - \alpha_{t} \label{eq:se-cond1}
	\end{align}
	On the event $E$,
	\begin{align}
	\mse_{A^*} \geq \widehat{\mse}_{A^*} - \alpha_{t}, \quad\widehat{\mse}_{A_{t}} + \alpha_{t} \geq \mse_{A_{t}}  \label{eq:se-cond2} 
	\end{align}   
	%\begin{align*}
	Substituting  \eqref{eq:se-cond2} in  \eqref{eq:se-cond1},  we obtain the following: $\mse_{A^*} \geq \widehat{\mse}_{A^*} - \alpha_{t} \geq \widehat{\mse}_{A_{t}} + \alpha_{t}  \geq \mse_{A_{t}}$,
	%\end{align*}
	and this leads to a contradiction. Hence, w.p. $1 - \delta$, the best subset is never eliminated, and the successive elimination algorithm is $\left( 0, \delta \right)$ - PAC.
	
	Next, we derive a bound on sample complexity of the successive elimination algorithm.
	
	Notice that, on the event $E$,  from \eqref{eq:se-cond2}, we have $\widehat{\mse}_{A^*} \leq \mse_{A^*} + \alpha_{t}, \widehat{\mse}_{A_{t}} \geq \mse_{A} - \alpha_{t}$.

	Now, $\widehat{\mse}_{A^*} + \alpha_{t} \leq \mse_{A^*} + 2 \alpha_{t} \leq \mse_{A} - 2 \alpha_{t} \leq \widehat{\mse}_{A_{t}} - \alpha_{t} \textrm{ which holds if } \mse_{A} - \mse_{A^*} \geq 4 \alpha_{t} $ or,\\
	equivalently  $\Delta(A) \geq 4 \alpha_{t}.$
	
	From \eqref{eq:alphat_appendix}, for some universal constants $c_1, c_2$, we have  
 \begin{align*}
     & \hspace{-2em }\alpha_{t} = \sqrt{a} \left(c_{2} \sqrt{a} + c_{1} \right), \textrm{ where } a =  \left(\frac{\log \left(\frac{2 \; \binom{K}{m} t^2}{\delta}\right)}{c_{1} \; t}\right).
 \end{align*}
	%\[\alpha_{t} = \sqrt{a} \left(c_{2} \sqrt{a} + c_{1} \right), \textrm{ where } a =  \left(\frac{\log \left(\frac{70 \; \binom{K}{m} K m^2 t^2}{\delta}\right)}{2 c_{3} \; t}\right).\] 
 \begin{align*}
     & \hspace{-3em} \textrm{Solving } \Delta(A) - 4 \left(c_{2} a + \sqrt{c_{1}} \sqrt{a} \right) \geq 0, \textrm{ we obtain,  } \textrm{as a solution for } a, \\
     & \qquad \qquad \qquad \qquad \qquad 0 \leq a \leq \frac{-2 c_{2}^2 \sqrt{\frac{c_{1} \left(c_{1} + c_{2} \Delta(A)\right) }{c_{2}^4}}  + 2c_{1} + c_{2} \Delta(A)}{4 c_{2}^2}.\\
     & \hspace{-3em} \textrm{Therefore, } a \leq \frac{2c_{1} + c_{2} \Delta(A)}{4 c_{2}^2}. \textrm{ Finally, by solving the equation } \left(\frac{\log \left(\frac{2 \; \binom{K}{m} t^2}{\delta}\right)}{ c_{1} \; t}\right) \leq \frac{2c_{1} + c_{2} \Delta(A)}{4 c_{2}^2}, \\
 \end{align*}
	%Solving $\Delta(A) - 4 \left(c_{2} a + \sqrt{c_{1}} \sqrt{a} \right) \geq 0 $, we obtain, as a solution for $a$, $0 \leq a \leq \frac{-2 c_{2}^2 \sqrt{\frac{c_{1} \left(c_{1} + c_{2} \Delta(A)\right) }{c_{2}^4}}  + 2c_{1} + c_{2} \Delta(A)}{4 c_{2}^2}$. Therefore, $a \leq \frac{2c_{1} + c_{2} \Delta(A)}{4 c_{2}^2}$. Finally, by solving the equation $ \left(\frac{\log \left(\frac{70 \; \binom{K}{m} K m^2 t^2}{\delta}\right)}{2 c_{3} \; t}\right) \leq \frac{2c_{1} + c_{2} \Delta(A)}{4 c_{2}^2} $, we obtain
	%\[t(A) = \mathcal{O} \left( \frac{1}{\Delta(A)} \log \left(\frac{ \binom{K}{m} \log \left(\Delta(A)^{-1}\right)}{\delta}\right)\right)\] 
 In each round of successive elimination, we observe a sample $K$-vector from the underlying multivariate Gaussian distribution.
 Therefore, w.p. $1 - \delta$, the overall sample complexity is bounded above by
 \begin{align*}
     &\mathcal{O} \left( \frac{1}{\Delta} \log \left(\frac{\binom{K}{m} \log \left(\Delta^{-1}\right)}{\delta}\right)\right),
 \end{align*}
where $\Delta$ denotes the smallest gap. 	%\todod{Add proof details that lead to the bound above from \eqref{eq:se-cond3}}
	%Refer \ref{sec:calc}
	%	$\left(\sum\limits_{A_{t} \in A} \frac{1}{\Delta_{A_{t}}^2} \log \left(\frac{\binom{K}{m} (Km^2)}{\Delta_{A_{t}}}\right)\right)$
\end{proof}

%%%%%%%%%%%%%%%%%%%%%%%%%%%%%%%%%%%%%%%%%%%%%%%%%%%%%%%%%%%%%%%%%%%%%%

%%%%%%%%%%%%%%%%%%%%%%%%%%%%%%%%%%%%%%%%%%%%%%%%%%%%%%%%%%%%%%%%%%%%%%
\section{Conclusions}
\label{sec:conclusions}
For the problem of estimation of the MSE of a given subset, with a multivariate Gaussian model, we proposed two  estimators, and derived tail bounds that exponentially concentrate. The first estimator is non-adaptive and exhibits weaker sample complexity guarantees in comparison to the second estimator, which is derived using a regression framework. We formulated the estimation problem with bandit feedback in the best-subset identification setting, and proposed a variant of the successive elimination technique. Finally, we also derived a minimax lower bound to understand the fundamental limit on the sample complexity of the aforementioned estimation problem with bandit feedback.

%%%%%%%%%%%%%%%%%%%%%%%%%%%%%%%%%%%%%%%%%%%%%%%%%%%%%%%%%%%%%%%%%%%%%%%%%%%%%%%%

\bibliography{refs}
\bibliographystyle{IEEEtran}
%\clearpage
%%%%%%%%%%%%%%%%%%%%%%%%%%%%%%%%%%%%%%%%%%%%%%%%%%%%%%%%%%%%%%%%%%%%%%%%%%%%%%%%
%\newpage

\end{document}